\documentclass[a4paper]{article}
\usepackage{fullpage}
\usepackage{natbib}
\usepackage[utf8]{inputenc}
\usepackage[T1]{fontenc}
\usepackage{amsmath,amssymb,amsfonts,amsthm,mathrsfs,bm,mathtools}
\usepackage{graphicx,booktabs,microtype,xcolor,enumitem}
\usepackage[ruled,vlined]{algorithm2e}
\usepackage{hyperref,url,prettyref}
\usepackage{thmtools,thm-restate}
\hypersetup{colorlinks=true,linkcolor=blue,citecolor=blue,urlcolor=blue}
\allowdisplaybreaks

\newtheorem{lemma}{Lemma}
\newtheorem{assumption}{Assumption}
\newtheorem{definition}{Definition}

\newcommand{\R}{\mathbb{R}}

\newcommand{\Reg}{\mathrm{Reg}}
\newcommand{\order}{\mathcal{O}}
\newcommand{\calX}{\mathcal{X}}
\newcommand{\cX}{\mathcal{X}}
\newcommand{\calG}{\mathcal{G}}
\newcommand{\calT}{\mathcal{T}}
\newcommand{\bx}{\mathbf{x}}
\newcommand{\bg}{\mathbf{g}}

\newcommand{\x}{\mathbf{x}}
\newcommand{\wh}{\widehat}

\newcommand{\inner}[1]{\left\langle #1\right\rangle}
\DeclareMathOperator*{\argmin}{arg\,min}

\newcommand{\pref}[1]{\prettyref{#1}}
\newcommand{\savehyperref}[2]{\texorpdfstring{\hyperref[#1]{#2}}{#2}}
\newrefformat{eq}{\savehyperref{#1}{Eq.~\textup{(\ref*{#1})}}}
\newrefformat{eqn}{\savehyperref{#1}{Eq.~\textup{(\ref*{#1})}}}
\newrefformat{lem}{\savehyperref{#1}{Lemma~\ref*{#1}}}
\newrefformat{def}{\savehyperref{#1}{Definition~\ref*{#1}}}
\newrefformat{line}{\savehyperref{#1}{Line~\ref*{#1}}}
\newrefformat{thm}{\savehyperref{#1}{Theorem~\ref*{#1}}}
\newrefformat{sec}{\savehyperref{#1}{Section~\ref*{#1}}}
\newrefformat{app}{\savehyperref{#1}{Appendix~\ref*{#1}}}
\newrefformat{asm}{\savehyperref{#1}{Assumption~\ref*{#1}}}
\newrefformat{alg}{\savehyperref{#1}{Algorithm~\ref*{#1}}}

\title{Decentralized Online Exp-Concave Optimization}
\author{
  \begin{tabular}{ccc}
    Hao Qiu$^{1}$\thanks{Equal contribution.} & Mengxiao Zhang$^{2}$\footnotemark[1] & Yuanyu Wan$^{3,4}$\thanks{Corresponding author.} \\
    {\small\texttt{qiuhaosai@gmail.com}} & {\small\texttt{mengxiao-zhang@uiowa.edu}} & {\small\texttt{wanyy@zju.edu.cn}}
  \end{tabular}\\[0.8em]
  {\small $^{1}$National University of Singapore, Singapore}\\[0.3em]
  {\small $^{2}$Department of Business Analytics, Tippie College of Business,}\\
  {\small University of Iowa, USA}\\[0.3em]
  {\small $^{3}$School of Software Technology, Zhejiang University, Ningbo, China}\\[0.3em]
  {\small $^{4}$State Key Laboratory of Blockchain and Data Security,}\\
  {\small Zhejiang University, Hangzhou, China}
}
\date{}
\begin{document}
\maketitle
\begin{abstract}
Decentralized Online Convex Optimization (D-OCO) is a fundamental framework in which a set of local agents aims to minimize a sequence of global loss functions using only local computations and communications. While recent work has established near-optimal regret guarantees for convex and strongly convex losses, the case of exp-concave losses remains an open challenge: specifically, whether logarithmic regret is achievable. In this paper, we answer this question affirmatively by proposing Decentralized  Online Newton Step (D-ONS), which adapts the idea of blocking and accelerated gossip to the Online Newton Step algorithm in a non-trivial manner. Concretely, our algorithm achieves a regret bound of $\mathcal{O}\left( Nd(1-\sigma_2)^{-1/2}\,\log^2(T) \right)$ where $N$ is the number of agents, $d$ is the dimension, $T$ is the horizon, and $1-\sigma_2$ represents the spectral gap of the network. To our knowledge, this is the first logarithmic regret bound for decentralized online exp-concave optimization. We further establish a nearly matching lower bound of $\Omega\left(Nd (1-\sigma_2)^{-1/2}\,\log(T/N)  \right)$ showing the near optimality of our upper bound. Finally, we extend our algorithm to decentralized  online linear regression in which the feasible domain is unbounded and features and labels are locally observed. By gossiping covariance and label-feature statistics, we also achieve logarithmic regret. 
\end{abstract}

\section{Introduction}
\label{sec:intro}
Decentralized online convex optimization (D-OCO) \citep{yan2012distributed, hosseini2013online, wan2024nearly} provides a powerful framework for sequential decision making in distributed systems, where multiple agents \emph{collaboratively} optimize a global objective while only observing \emph{local} information and communicating through a network. In this setting, agents repeatedly choose decisions, incur global losses but only observe local feedback, and exchange information with neighboring agents in order to compete with the best fixed decision for the global losses in hindsight. This paradigm naturally captures a variety of applications including multi-robot systems \citep{shorinwa2024distributed}, sensor networks \citep{li2002detection}, online packet routing \citep{awerbuch2004adaptive}, where centralized coordination is either infeasible due to communication constraints or undesirable due to privacy and robustness considerations.

For general convex losses, early studies by \citet{yan2012distributed, hosseini2013online} established decentralized online algorithms with regret bounds $\mathcal{O}(N^{5/4}(1-\sigma_2)^{-1/2}\sqrt{T})$ and $\mathcal{O}(N^{3/2}(1-\sigma_2)^{-1}\log(T))$ for convex and strongly convex losses, respectively. Here, $N$ denotes the number of agents, $T$ is the horizon, and $1-\sigma_2$ represents the spectral gap of the communication matrix. More recently, \citet{wan2024nearly, wan2025optimal} tightened these guarantees using accelerated gossip and block-based scheduling, achieving $\mathcal{O}(N(1-\sigma_2)^{-1/4}\sqrt{T \log(N)} )$ regret for convex losses and $\mathcal{O}(N(1-\sigma_2)^{-1/2}\,\log(N) \log(T))$ for strongly convex losses, alongside nearly matching lower bounds.

Despite these advances, the landscape for decentralized OCO with exp-concave losses remains largely unexplored. This class of functions captures a wide range of applications, including online portfolio selection, logistic regression, and online linear regression on bounded domains. Similar to the strongly convex case, exp-concavity allows for significantly faster convergence rates in centralized settings than general convexity. Specifically, the seminal Online Newton Step (ONS) algorithm~\citep{hazan2007logarithmic} achieves $\mathcal{O}(d \log(T))$ regret for $d$-dimensional exp-concave losses. However, the possibility of achieving such fast rates with the decentralized communication remains an open challenge. This leads to the following central question of this work:

\begin{quote}
\centering
\emph{Can decentralized agents exploit exp-concavity to achieve logarithmic regret while communicating only through local gossip?}
\end{quote}

\paragraph{Our contributions.}
In this paper, we answer the above question affirmatively.  More concretely:

\begin{itemize}[leftmargin=*,nosep]
    \item We propose Decentralized Online Newton Step (D-ONS), a decentralized variant of Online Newton Step for exp-concave losses under a gossip communication protocol. Specifically, we obtain a regret bound of 
    $
    \mathcal{O}(
    N d 
    (1-\sigma_2)^{-1/2}\,\log^2(T)
    ).$
    To the best of our knowledge, this is the first algorithm to achieve logarithmic regret for exp-concave losses under D-OCO,  answering an open question proposed in  \citet{wan2025black}. Technically, our algorithm combines three ingredients: a blocking scheme that allows communication to be amortized over time, an accelerated gossip protocol to obtain a sharp dependency on the network spectral gap, and a delayed ONS update adapted to the exp-concave structure. 
    \item  We further complement our upper bound with a nearly matching lower bound, demonstrating that the polynomial dependence on the number of agents, the dimension, and the network spectral gap is essentially unavoidable. In particular, we adapt the hard instance of online exp-concave optimization from~\citet{ordentlich1998cost} to the decentralized setting. This establishes that no algorithm can achieve a regret better than $\Omega\left(Nd(1-\sigma_2)^{-1/2}\,\log(T/N)\right)$, showing that our upper bound is near-optimal up to logarithmic factors.

    \item Finally, we extend our approach to decentralized online linear regression, where the decision domain for each agent is \emph{unbounded} and thus not captured by standard exp-concave analysis. Specifically, we propose a decentralized online ridge regression algorithm with clipped predictions that achieves a regret bound of  $
    \mathcal{O}(
    N d 
    (1-\sigma_2)^{-1/2}\,\log^2(T)
    )$, demonstrating that logarithmic regret remains attainable even in the absence of a compact feasible set in online linear regression.
\end{itemize}

\subsection{Related Works}
\paragraph{Decentralized online convex optimization.}
Decentralized online convex optimization (D-OCO) studies sequential decision making over a network, where multiple agents collaboratively minimize a global objective without relying on a central coordinator. At each round, agents observe only local loss information and communicate with neighboring agents through a graph, making consensus and gossip-based methods central algorithmic tools. Foundational work in decentralized optimization was developed primarily in offline settings, where a standard gossip step \cite{xiao2004fast} was used to enable decentralized consensus and optimization \citep{nedic2009distributed,duchi2011dual}. The online counterpart was first systematically studied by \citet{yan2012distributed}, who proposed a Decentralized Online Gradient Descent and established sublinear regret guarantees. In particular, their algorithm achieves a regret bound of $\mathcal{O}(N^{5/4} (1-\sigma_2)^{-1/2}\sqrt{T})$. More recently, \citet{wan2024nearly, wan2025optimal} leverage an accelerated gossip strategy \citep{liu2011accelerated,ye2023multi} together with a blocking update mechanism to develop new D-OCO algorithms that achieve regret bounds of $\mathcal{O}(N(1-\sigma_2)^{-1/4}\sqrt{T \log(N)})$ for convex functions and $\mathcal{O}(N(1-\sigma_2)^{-1/2}\,\log(N)\log(T))$ for strongly convex functions. They further establish matching lower bounds of $\Omega(N(1-\sigma_2)^{-1/4}\sqrt{T})$ and $\Omega(N (1-\sigma_2)^{-1/2} \log(T/N))$ for the convex and strongly convex settings, respectively, demonstrating the near-optimality of their algorithms. See \citet{yuan2024multi} for a comprehensive overview.

\paragraph{Online learning with curved losses.}
Curvature assumptions on the loss function are well known to yield faster regret rates in online convex optimization. For strongly convex losses with $G$-bounded gradients, Online Gradient Descent achieves $\mathcal{O}(G^2 \log(T))$ regret, while for $d$-dimensional exp-concave losses, the Online Newton Step of \citet{hazan2007logarithmic} attains $\mathcal{O}(d \log(T))$ regret. Closely related second-order methods for online linear regression include the Vovk--Azoury--Warmuth forecaster~\citep{vovk2001competitive} and online ridge regression \citep{hoerl1970ridge}. Beyond these foundations, research has expanded into several advanced perspectives, including scale-free guarantees \citep{pmlr-v178-mayo22a}, non-stationarity \citep{zhang2025non}, adaptive guarantees \citep{bartlett2007adaptive, pmlr-v178-luo22a}, and robustness to delayed feedback \citep{qiu2025exploiting}. In contrast to these established paradigms, our work investigates exp-concave online learning in a decentralized setting, where multiple agents must collaboratively minimize the global objective over a communication network. For a more extensive literature review on online learning with curved losses, we refer the reader to several comprehensive surveys \citep{hazan2016introduction, orabona2026}.

\section{Preliminaries}\label{sec:pre}

Throughout this paper, we denote the set $\{1,2,...,m\}$ for some positive integer $m$ by $[m]$ and let $\mathbf{1}$ be an all-one vector in an appropriate dimension. For a matrix $M\in\R^{m_1\times m_2}$, denote its $(i,j)$-th entry by $M_{ij}$.

In this work, we study decentralized online convex optimization (D-OCO). 
The framework consists of a group of $N$ agents interacting with a dynamic environment and communicating over a network modeled by a connected, undirected graph 
$\mathcal G\triangleq (V,E)$. The vertex set $V=[N]$ represents the agents, while the edge set $E$ specifies the available communication links. For each agent $i\in V$, we denote its neighborhood by
$
\mathcal{N}_i \triangleq \{j\in V : (i,j)\in E\}\cup\{i\}.
$
Throughout the paper, we use $V$ and $[N]$ interchangeably. Each agent $i \in [N]$ encounters a sequence of local loss functions $\{f_{t,i}\}_{t=1}^T$, where each $f_{t,i}: \mathcal{X} \subset \mathbb{R}^d  \to \mathbb{R}$ is chosen arbitrarily by an adversary and is unknown to the agents. Below, we list several standard assumptions \citep{yan2012distributed, hosseini2013online, wan2024nearly} and will explicitly specify which are required for each of our results. 

\begin{assumption}
\label{asm: cvx}
For each round $t \in [T]$, the local loss function $f_{t,i}$ of each agent $i \in [N]$ is $G$-Lipschitz continuous over $\mathcal{X}$ with respect to the $\ell_2$ norm, i.e., $|f_{t,i}(\bx) - f_{t,i}(\mathbf{y})| \leq G \|\bx - \mathbf{y}\|_2$ for all $\bx, \mathbf{y} \in \mathcal{X}$.
\end{assumption}

\begin{assumption}
\label{asm: domain}
The feasible set $\mathcal{X}$ is compact and contains the origin ($\mathbf{0} \in \mathcal{X}$). Furthermore, its diameter is bounded by $D$, such that $\|\bx - \mathbf{y}\|_2 \leq D$ for any $\bx, \mathbf{y} \in \mathcal{X}$.
\end{assumption}

In this paper, we mainly focus on \emph{exp-concave losses} \cite{hazan2007logarithmic}, a function class that is broader than the class of strongly convex functions and also enjoys faster convergence rate in classic online optimization. 
\begin{assumption}
\label{asm: expconcave}
For each round $t \in [T]$, the loss function $f_{t,i}:\calX\mapsto\R$ of each agent $i \in [N]$ is $\alpha$-exp-concave over $\mathcal{X}$. That is, the mapping $\bx \mapsto \exp(-\alpha f_{t,i}(\bx))$ is concave for all $\bx \in \mathcal{X}$.
\end{assumption}

The learning protocol of D-OCO involves an interaction between the agents and the environments over $T$ rounds. Without loss of generality, we assume that $T\geq N$. At each round $t \in [T]$, each agent $i \in [N]$ chooses a decision $\bx_{i}(t)$ from the feasible set $\calX$. Simultaneously, the adversary selects the local loss functions $f_{t,i}$ for all $i \in [N]$. Each agent $i$ then observes only the gradient of her \emph{local loss} $\nabla f_{t,i}(\bx_{i}(t))$. Subsequently, agents are permitted to exchange information with their neighbors in the graph $\calG$ to update their internal states for the next round. The objective of each agent $i$ is then to minimize the regret in terms of \emph{the global function}, defined as:
\begin{equation} 
    \Reg_{T,i} \triangleq \max_{\mathbf{x} \in \mathcal{X}}\Reg_{T,i}(\mathbf{x}) \triangleq \max_{\bx \in \mathcal{X}}\left(\sum_{t=1}^T \sum_{j=1}^N f_{t,j}(\bx_{i}(t)) -  \sum_{t=1}^T \sum_{j=1}^N f_{t,j}(\bx)\right)~.
\end{equation}

It remains to specify how agents communicate over the network $\mathcal{G}$. Following \citet{hosseini2013online, yan2012distributed, wan2024nearly}, we model information exchange through a gossip protocol induced by a communication matrix $P \in \mathbb{R}^{N \times N}$. Intuitively, $P$ encodes how each agent aggregates messages from its neighbors: the entry $P_{ij}$ represents the weight that agent $i$ assigns to the information received from agent $j$. Since communication is restricted by the network structure, $P$ is supported on $\mathcal{G}$, and its symmetry and double stochasticity ensure that repeated local averaging preserves the global mean. We formalize these requirements below.

\begin{definition}
\label{def:5}
The communication matrix $P \in \mathbb{R}^{N \times N}$ is supported on the graph 
$\mathcal{G}=([N],E)$, symmetric, and doubly stochastic:
(i) $P_{ij}>0$ only if $(i,j)\in E$ or $i=j$;
(ii)$\sum_{j=1}^N P_{ij}=1$ for all $i\in[N]$;
(iii) $\sum_{i=1}^N P_{ij}=1$ for all $j\in[N]$.
Moreover, $P$ is positive semidefinite, and its second largest singular value 
$\sigma_2(P)$ satisfies $\sigma_2(P)<1$. 
When the dependence on $P$ is clear, we
write $\sigma_2$. Here, for a matrix $P$, we denote by 
$\sigma_1(P) \geq \sigma_2(P) \geq \cdots$ its singular values in nonincreasing order.
\end{definition}

Given a collection of local vectors $\{\mathbf{u}_i\}_{i=1}^N$ (for example, local gradients or iterates), the goal of gossip is to approximately recover the network average $\overline{\mathbf{u}}=\frac{1}{N}\sum_{i=1}^N \mathbf{u}_i$ for every agent using only neighbor-to-neighbor communication. To this end, agents run $L$ rounds of gossip. In each round $k$, agent $i$ updates its local state according to 
\begin{equation}
    \mathbf{u}_i^{k+1} = (1+\theta)\sum_{j\in \mathcal{N}_i} P_{ij}\mathbf{u}_j^{k} - \theta \mathbf{u}_i^{k-1},
\end{equation}
for $k \ge 0$, where $\mathbf{u}_i^0 = \mathbf{u}_i^{-1} = \mathbf{u}_i$. When $\theta=0$, this reduces to the standard gossip iteration \citep{xiao2004fast}; when $\theta>0$, it becomes the accelerated gossip scheme of \citet{liu2011accelerated}, which improves the convergence rate to $\overline{\mathbf{u}}$.

The following lemma extended from \citet{ye2023multi} shows that accelerated gossip allows each local vector to approximate the global average vector with an error of order $\mathcal{O}((TN)^{-2})$ after $L=\Theta((1-\sigma_2)^{-1/2}\,\log(NT))$ by picking $\theta$ appropriately. We include the proof in \pref{sec: app_aux} for completeness.

\begin{restatable}{lemma}{gossip}
\label{lem:gossip-close-exp}
Let $P$ be the communication matrix defined in \pref{def:5}. Let
\begin{align}\label{eqn:para-choice}
    \theta\triangleq
    \frac{1}{1+\sqrt{1-\sigma_2^2(P)}} , \qquad
    L\triangleq
    \left\lceil
    \frac{\sqrt{2}\log(T^2\sqrt{14} N^{5/2})}
    {(\sqrt{2}-1)\sqrt{1-\sigma_2(P)}}
    \right\rceil.
\end{align}
Consider arbitrary vectors
$
    \mathbf{u}_1,\ldots,\mathbf{u}_N\in\mathbb{R}^d
$
satisfying $
    \|\mathbf{u}_i\|_2\le R, i\in[N].
$
Let
\begin{align*}
    U^0
    = U^{-1} =
    \begin{bmatrix}
    \mathbf{u}_1^\top\\
    \vdots\\
    \mathbf{u}_N^\top
    \end{bmatrix},
    \qquad
    \overline{\mathbf{u}}
    =
    \frac{1}{N}\sum_{i=1}^N \mathbf{u}_i^0,
    \qquad
    \overline U
    =
    \begin{bmatrix}
    \overline{\mathbf{u}}^\top\\
    \vdots\\
    \overline{\mathbf{u}}^\top
    \end{bmatrix}.
\end{align*}
Suppose accelerated gossip is run for $L$ iterations with update $U^{k+1}
    =
    (1+\theta)PU^k-\theta U^{k-1}$ for $k\geq 0$. Then, for every agent $i\in[N]$,
    $\|\mathbf{u}_i^L-\overline{\mathbf{u}}\|_2
    \le
    \frac{2R}{T^2N^2}.$
\end{restatable}

\paragraph{Other notations.}  
Define $\Delta_d\triangleq\{x\in \R^d:\sum_{j=1}^dx_j=1, x_j\geq 0, \forall j\in[d]\}$ to be the $(d-1)$-dimensional simplex. \(I_d\) denotes the \(d \times d\) identity matrix. Let $A_{\mathcal{G}}$ be the adjacency matrix of $\mathcal{G}$. Let
$D_{\mathcal{G}}=\operatorname{diag}(d_1,\ldots,d_N)$ be the diagonal degree matrix, and let
$d_{\max}=\max_{i\in[N]} d_i$ denote the maximum degree of $\mathcal{G}$.
The graph Laplacian of $\mathcal{G}$ is defined as
$
L_{\mathcal{G}} \triangleq D_{\mathcal{G}} - A_{\mathcal{G}}.
$ For a positive semidefinite matrix $M \in \mathbb{R}^{d \times d}$ and $x \in \mathbb{R}^d$, we denote $\|x\|_M=\sqrt{x^{\top} M x}$ to be the Mahalanobis norm induced by $M$ and, if $M$ is positive definite, let $\|x\|_{M^{-1}}=\sqrt{x^{\top} M^{-1} x}$ be the dual norm.

\section{Decentralized Online Exp-concave Optimization}
\label{sec:d-exp-concave}
In this section, we introduce our algorithm for decentralized online exp-concave optimization. We first provide an algorithm that achieves a logarithmic regret upper bound (\pref{sec:upper_exp}), followed by a nearly-matching lower bound (\pref{sec: lower_exp}) to show the optimality of our algorithm.

\subsection{Upper Bound}\label{sec:upper_exp}
In this subsection, we present our algorithm for decentralized online exp-concave optimization. 
Our algorithm builds upon the block-based communication schedule and accelerated gossip mechanism developed in recent decentralized OCO reductions for the convex and strongly convex settings~\citep{wan2024nearly,wan2025black}. However, the exp-concave case cannot be handled by a direct black-box reduction, as unlike the strongly convex setting where the curvature is characterized by a fixed global parameter, the curvature information in exp-concave optimization is decision-dependent and may incur larger consensus errors during communication.
To address this challenge, we propose Decentralized Online Newton Step (D-ONS), which gossips both block gradients and decisions before performing a delayed ONS update.

\begin{algorithm}[ht]
\caption{Decentralized Online Newton Step for each agent $i \in [N]$}
\label{alg:dons}
\nl \textbf{Input:} gossip parameter $\theta$, epoch length $L$, regularization parameter $\beta$, $\eta$, communication matrix $P$

\nl \textbf{Initialization:} set $\bx_i(1)=\hat \bx_i(1)=\mathbf{0}$.

\nl \For{$z=1,\dots,T/L$}{

    \nl Compute
    \begin{align}\label{eqn:update-exp}
        \hat{\bx}_i(z+1)
            =
            \argmin_{\bx\in\calX}
            \sum_{s=1}^{z-2}
            \left(
            \inner{\bg_i(s),\bx}
            +
            \frac{\beta}{2L}\inner{\bg_i(s), \bx_i(s)-\bx}^2 
            \right)
            +
            \frac{\eta}{2}\|\bx\|_2^2.
    \end{align} \label{line: ons}       
            
        \nl Set $k_x=0$ and $k_g=0$ and $\bx_i^0(z+1)=\bx_i^{-1}(z+1)=\hat{\bx}_i(z+1)$ 
    
    \nl \For{$t=(z-1)L+1,\dots,zL$}{

        \nl Play $\bx_i(z)$ and query $\nabla f_{t,i}(\bx_i(z))$ \label{line: decision-exp}       

        \nl \lIf{ $z \ge 2$}{ update
        \(   \bg_i^{k_g+1}(z-1)
        =
        (1+\theta)\sum_{j\in \mathcal{N}_i} P_{ij}\bg_j^{k_g}(z-1)
        -
        \theta \bg_i^{k_g-1}(z-1)
        \)} \label{line:gradient_gossip}

        \nl Update
        \(
        \bx_i^{k_x+1}(z+1)
        =
        (1+\theta)\sum_{j\in \mathcal{N}_i} P_{ij}\bx_j^{k_x}(z+1)
        -
        \theta \bx_i^{k_x-1}(z+1)
        \) \label{line:decision_gossip}
        
        \nl Set $k_x \leftarrow k_x+1$ and $k_g \leftarrow k_g+1$
    }

    \nl Update
    $
    \bx_i(z+1)
    =
    \argmin_{\bx\in\calX}\|\bx_i^L(z+1)-\bx\|_2
    $
    
    \nl \lIf {$z \ge 2$}{set $\bg_i(z-1)=\bg_i^L(z-1)$}

    \nl Set
    $
    \bg_i^0(z)=\bg_i^{-1}(z)
    =
    \sum_{\tau \in \mathcal{T}_z}\nabla f_{\tau,i}(\bx_i(z))
    $
}
\end{algorithm}

More specifically, the algorithm partitions the time horizon $T$ into $T/L$ blocks indexed
by $z$, where $
    \mathcal{T}_z
    \triangleq
    \{(z-1)L+1,\ldots,zL\}$
denotes the set of time steps in block $z$.\footnote{For notational convenience, we assume that $T/L$ is an integer.} Before introducing the algorithm, for notational convenience, we define $\overline{\bg}(z)$ and $\overline{\bx}(z)$ as follows:
\begin{align}\label{eqn:avg-definition}
    \overline{\bg}(z)\triangleq\frac{1}{N} \sum_{t\in \calT_z}\sum_{j=1}^N \nabla f_{t, j}\left(\mathbf{x}_j(z)\right),~~~~~\overline{\mathbf{x}}(z) \triangleq \frac{1}{N}\sum_{j=1}^N \hat{\mathbf{x}}_j(z).
\end{align}
Specifically, 
$\overline{\mathbf{g}}(z)$ denotes the network average of the cumulative gradients collected during block \(z\), while $\overline{\mathbf{x}}(z)$ denotes the average of the preparatory decisions. \pref{alg:dons} proceeds over these blocks as follows. 

\begin{itemize}[leftmargin=*,nosep]
    \item \textbf{Step 1: Local Preparatory Update.} Before making a decision for the current block $z$, each agent $i$ computes a local preparatory decision $\hat{\mathbf{x}}_i(z+1)$. This is obtained via a delayed ONS update (\pref{line: ons}) using the gossiped aggregated gradient $\mathbf{g}_i(s)$ and decision information $\bx_i(s)$ collected from blocks $s=1$ to $z-2$.

    \item \textbf{Step 2: Decision Execution and Parallel Gossip.} Each agent $i$ plays the fixed decision $\mathbf{x}_i(z) \in \mathcal{X}$ for all rounds $t \in \mathcal{T}_z$ and observes the local gradients $\nabla f_{t,i}(\mathbf{x}_i(z))$ (\pref{line: decision-exp}). Simultaneously, two accelerated gossip processes are executed in parallel:
\begin{itemize}[leftmargin=*,nosep]
    \item \textit{Gradient Gossip:} Agents run accelerated gossip on the gradients collected during the previous block $z-1$ (\pref{line:gradient_gossip}). According to \pref{lem:gossip-close-exp}, by the end of block $z$, the local state $\mathbf{g}_i^L(z-1)$ will approximate the network average of the cumulative gradients $\overline{\mathbf{g}}(z-1)$ within a $\text{poly}(1/T)$ error. This will serve as the feedback for future ONS updates.
    \item \textit{Decision Gossip:} Agents also run accelerated gossip on their local preparatory decisions $\hat{\mathbf{x}}_i(z+1)$. While one might intuitively use $\hat{\mathbf{x}}_i(z+1)$ directly as the decision, these local ONS outputs may \emph{diverge} significantly across agents, complicating the regret analysis. To ensure consensus on the decision over all the agents, we also gossip these vectors so that $\mathbf{x}_i^L(z+1)$ converges toward the network average $\overline{\mathbf{x}}(z+1)=\frac{1}{N}\sum_j \hat{\mathbf{x}}_j(z+1)$ up to a $\text{poly}(1/T)$ consensus error.
\end{itemize}

    \item \textbf{Step 3: Projection and Finalization.} At the end of block $z$, the output of the decision-gossip process $\mathbf{x}_i^L(z+1)$ is projected back onto the feasible set $\mathcal{X}$. This projection is necessary because the linear combinations involved in accelerated gossip may produce iterates that lie outside the convex constraints.\footnote{We point out that the algorithm in \citet{wan2025black} in fact also requires this projection to make sure that $\bx_i(z)$ lies in the feasible set.} The resulting projected vector is then used as the played decision $\mathbf{x}_i(z+1)$ for the next block. This ensures that the decisions made by each agent at each epoch are close to each other. For $z \ge 2$, the refined gradient estimate $\mathbf{g}_i(z-1)$ is fed into the local ONS subroutine to guide the next iteration.
\end{itemize}

Consequently, our algorithm can be viewed as operating with a constant \emph{2-block delay}. Specifically, after agents collect their local cumulative gradients in block $z$, they perform a gossip process during block $z+1$ to approximate the network average of the cumulative gradients. Following the ONS update, a second accelerated gossip process is executed in block $z+2$ to synchronize the decisions for use in block $z+3$. Although this delayed information flow is adapted from the blocking framework of \citet{wan2025black}, the analysis must account for the gradient-dependent curvature of exp-concave losses and the distinct preparatory and executed decisions.

Formally, we state our regret guarantee for  decentralized  online exp-concave optimization as follows:

\begin{restatable}{theorem}{mainexpconcave}
\label{thm:upper-exp}
Suppose that the sequence of loss functions $\{f_{t,i}\}_{t \in[T], i\in[N]}$ satisfy \pref{asm: cvx}, \ref{asm: domain}, and \ref{asm: expconcave}. Let $P$ be the communication matrix defined in \pref{def:5}. If every agent $i\in[N]$ runs \pref{alg:dons} with $\beta=\frac{1}{2}\min\{\frac{1}{4GD}, \alpha\}$ and with parameters $\theta$ and $L$ chosen in \pref{eqn:para-choice},
then, for any agent $i\in[N]$, the regret is bounded by:
\begin{equation}\label{eqn:reg-exp-result}
   \Reg_{T,i}(\bx) = \mathcal{O}\left( \eta N\|\bx\|_2^2 + \frac{Nd(1+\beta GD)L}{\beta} \log\left( 1+\frac{\beta G^2T}{\eta d} \right) \right).
\end{equation}
\end{restatable}

By substituting the choice of $L$ from \pref{eqn:para-choice} into \pref{eqn:reg-exp-result}, we obtain a total regret of $\mathcal{O}\left(Nd (1-\sigma_2(P))^{-1/2}\,\log^2(T)  \right)$. To our knowledge, this is the first logarithmic regret guarantee for decentralized online exp-concave optimization. We complement this result with a lower bound of $\Omega\left( Nd (1-\sigma_2(P))^{-1/2}\,\log(T/N) \right)$ in \pref{sec: lower_exp}, demonstrating the near-optimality of our upper bound with respect to $T$, the number of agents $N$, the dimension $d$, and the network spectral gap $1-\sigma_2(P)$.

\subsection{Analysis of~\pref{alg:dons}}
\label{sec:analysis-upper-exp}
In this subsection, we provide a proof sketch of \pref{thm:upper-exp}.
The complete proof is deferred to \pref{app:appexpconcave}. The analysis consists of two main steps. First, we provide a decentralized-to-block-level reduction, showing that the regret can be controlled by a sum of local ONS-type surrogate terms up to lower-order gossip errors.
For any comparator \(\mathbf{x}\in\mathcal X\), utilizing the Lipschitz property, the property of exp-concavity, and the consensus guarantee in \pref{lem:gossip-close-exp}, we obtain that
\begin{align}
f_{t,j}\bigl(\mathbf{x}_i(z)\bigr)-f_{t,j}(\mathbf{x})
&\le
\left\langle
\nabla f_{t,j}\bigl(\mathbf{x}_j(z)\bigr),
\overline{\mathbf{x}}(z)-\mathbf{x}
\right\rangle
-
\frac{\beta}{2}
\left\langle
\nabla f_{t,j}\bigl(\mathbf{x}_j(z)\bigr),
\mathbf{x}_j(z)-\mathbf{x}
\right\rangle^2 \nonumber \\ 
&\qquad+
\order\big(G\|\overline{\bx}(z)-\bx_i(z)\|+\|\overline{\bx}(z)-\bx_j(z)\|\big).
\label{eq: regret-de-1-main}
\end{align}
Specifically, the last term of \pref{eq: regret-de-1-main} can be bounded by $\order(GD/(TN)^2)$ since \(\bx_i(z)\) and \(\bx_j(z)\) are both projected approximations of the network-average decision \(\overline{\bx}(z)\), thus \pref{lem:gossip-close-exp} implies that their deviations from \(\overline{\bx}(z)\) are of order \(\order(D/(T^2N^2))\). Therefore, the last term in \pref{eq: regret-de-1-main} contributes only a lower-order term after summing over all rounds and agents.
For the first two terms, after taking a summation over all $j\in[N]$ and $t\in\calT_z$  and applying Jensen's inequality, with some calculation, we know that they are bounded by
\begin{align}\label{eqn:inter-main-exp}
    N\left\langle
        \overline{\mathbf{g}}(z),
        \overline{\mathbf{x}}(z)-\mathbf{x}
    \right\rangle
    -
    \frac{\beta}{2NL}
    \Big(
        \sum_{t\in\mathcal T_z}\sum_{j=1}^N
        \left\langle
            \nabla f_{t,j}(\mathbf{x}_j(z)),
            \mathbf{x}_j(z)-\mathbf{x}
        \right\rangle
    \Big)^2.
\end{align}

By definition of $\overline{\mathbf{x}}(z)$, the first term of \pref{eqn:inter-main-exp} can be written as $\sum_{j=1}^N\inner{\overline{\bg}(z),\wh{\bx}_j(z)-\bx}$, which can be further bounded by $\sum_{j=1}^N\inner{{\bg}_j(z),\wh{\bx}_j(z)-\bx}$ plus an $\order(1/T)$ approximation error due to \pref{lem:gossip-close-exp}. Via some calculation, the second term of \pref{eqn:inter-main-exp} can be bounded by $-\frac{\beta}{2L}
    \sum_{j=1}^N
    \left\langle
        \mathbf{g}_j(z),
        \mathbf{x}_j(z)-\mathbf{x}
    \right\rangle^2 + \order(\sum_{j=1}^N(\|\bx_j(z)-\overline{\bx}(z)\|_2+\|\bg_j(z)-\overline{\bg}(z)\|_2))$, where the second term is bounded by $\order(1/T)$ due to \pref{lem:gossip-close-exp} again. Combining all of the above, we reduce bounding $\Reg_{T,i}(\bx)$ to bounding
\begin{align}\label{eqn:reduced-term-to-bound}
    \sum_{z=1}^{T/L}\sum_{j=1}^N
    \left(
    \inner{\bg_j(z),\wh{\bx}_j(z)-\bx}
    -
    \frac{\beta}{2L}
    \inner{\bg_j(z),\bx_j(z)-\bx}^2
    \right)
\end{align}
    plus $\order(1)$ term induced by the gossip process. 
The second step is to bound \pref{eqn:reduced-term-to-bound}. 
A natural approach would be to interpret each summand as a surrogate loss and then apply a standard FTRL or ONS regret bound to the preparatory decision \(\widehat{\mathbf{x}}_j(z)\). 
However, unlike the strongly convex setting considered in~\citet{wan2025black}, such a direct argument does not apply here. 
The difficulty is that the linear term is evaluated at the preparatory decision \(\widehat{\mathbf{x}}_j(z)\), whereas the negative quadratic term depends on the executed decision \(\mathbf{x}_j(z)\). 
This mismatch prevents the summand in \pref{eqn:reduced-term-to-bound} from being treated as the value of a standard ONS surrogate loss at a single decision point. 

To overcome this issue, we analyze the delayed FTRL dynamics for \(\widehat{\mathbf{x}}_j(z)\) more carefully. 
The key observation is that although \(\widehat{\mathbf{x}}_j(z)\) is computed using only feedback up to block \(z-3\), it remains close, in the adaptive second-order geometry induced by the accumulated gradients, to the ``current'' regularized leader that has access to feedback up to block \(z\). This closeness allows us to control the price of the two-block delay by a logarithmic term using ellipsoid potential arguments.
The following lemma formalizes this delayed ONS bound. Its proof uses the same local-norm stability and elliptical-potential tools as \citet[Theorem~4.1]{qiu2025exploiting}; in particular, their Lemma~A.2 gives the intermediate stability comparison. Here, $\bg_j(z)$ aggregates gradients evaluated at different agents' executed decisions, while the linear and quadratic terms involve $\widehat{\bx}_j(z)$ and $\bx_j(z)$, respectively. We use the lemma to examine this two-decision structure, rather than highlight it as a separate novel contribution. The self-contained derivation is retained in \pref{app:appexpconcave}.

\begin{restatable}{lemma}{lemdelayedons}
\label{lem:block-ons-exp}
Under the same conditions as \pref{alg:dons}. For any $j\in[N]$ and every comparator $\bx\in\calX$,
\begingroup\medmuskip=2mu\thickmuskip=3mu
\begin{align*}
    &\sum_{z=1}^{T/L}
    \big(
    \langle\bg_j(z),\hat{\bx}_j(z)-\bx\rangle
    -
    \tfrac{\beta}{2L}
    \langle\bg_j(z),\bx_j(z)-\bx\rangle^2
    \big)\leq
    \mathcal{O}\Big(
    \eta\|\bx\|_2^2
    +
    \frac{d(1+\beta GD)L\log\big(
    1+\frac{\beta G^2T}{\eta d}
    \big)}{\beta}
    \Big).
\end{align*}
\endgroup
\end{restatable}
The full proof is deferred to \pref{app:appexpconcave}.
Applying \pref{lem:block-ons-exp} to each \(j\in[N]\) and combining it with the regret decomposition above yields the claimed bound in \pref{thm:upper-exp}. This completes the proof sketch.

\subsection{Lower Bound}\label{sec: lower_exp}
In this section, we provide the lower bound of decentralized online exp-concave optimization. Specifically, we show that no algorithm can achieve a regret better than $\Omega(Nd(1-\sigma_2)^{-1/2}\,\log(T/N))$ in decentralized exp-concave online optimization.

\begin{restatable}{theorem}{mainlowerbound}
\label{thm:lowerbound}
Let \(N=4K+2\) for some positive integer \(K\), and let
$\calX=\Delta_d$.\footnote{To align with \pref{asm: domain} used in our upper bound, one can shift the simplex $\Delta_d$ by $\frac{1}{d}\mathbf{1}$ to include the origin. The $\alpha$-exp-concavity is invariant under such constant shifts.} There exists a graph $\mathcal{G}=([N],E)$ with $P=I-\frac{1}{2(d_{\max}+1)} L_{\mathcal{G}} $ such that there exists a sequence of local convex
\(\alpha\)-exp-concave losses $\{f_{t,i}:\calX\to\mathbb R\}_{i\in [N],t\in[T]}$ such that
\(
    \mathbb E[\Reg_{T,1}]
    \ge
    \Omega\Big(\frac{Nd}{\alpha \sqrt{\left(1-\sigma_2(P)\right)}} \log(T/N)\Big).
\)
\end{restatable}

Compared to this lower bound, our upper bounds are optimal in their polynomial dependence on $1 - \sigma_2(P)$, $d$, $T$, and $N$. While a logarithmic gap in $T$ and $N$ remains, we leave bridging this gap as a future direction.

We provide a proof sketch here and defer the full proof to \pref{app: lower_exp}. 
Our proof is adapted from the construction in~\citet{wan2024nearly}. We consider a carefully designed instance on a $1$-connected cycle graph. The horizon is divided into blocks of length $\Theta(N)$. In each block, we place one hard $\alpha$-exp-concave portfolio loss from~\citet{ordentlich1998cost} on a contiguous middle segment of $\Theta(N)$ agents, while the remaining agents have identically zero local losses. Hence, over each block, the global loss is a $\Theta(N^2)$-scaled copy of one hard exp-concave loss.

Focusing on an agent $u$ in the zero-loss segment, we observe that all agents carrying the non-zero loss are at graph distance $\Theta(N)$ from $u$. Therefore, the information about the current block loss cannot reach $u$ before the block ends. After averaging the decisions of $u$ within each block, its optimization problem reduces to a centralized online exp-concave optimization instance on a sequence of $\Theta(T/N)$ losses. By convexity, the distributed regret of agent $u$ is at least $\Theta(N^2)$ times the regret in this centralized hard instance. Applying the portfolio lower bound of~\citet{ordentlich1998cost} gives
$
    \Omega( \alpha^{-1}N^2 d\log(T/N)).
$
Finally, for the cycle graph, the spectral gap satisfies
$
    1/\sqrt{1-\sigma_2(P)}=\Theta(N),
$
which yields the claimed lower bound.

\section{Decentralized Online Linear Regression}
\label{sec:dolr}
\setcounter{AlgoLine}{0}

In this section, we investigate decentralized online linear regression (D-OLR) over an unbounded domain. The learning protocol proceeds as follows. At each round $t\in[T]$, before making its prediction, agent $i$ observes only its local feature vector $\mathbf h_{t,i}$. Each agent $i$ decides a model $\bx_i(t)\in\mathbb{R}^d$ to generate its prediction. Crucially, after making this decision, agent $i$ only observes its local true label $y_{t,i}\in[-Y,Y]$ corresponding to the feature vector $\mathbf{h}_{t,i}$, while the features and labels of other agents are not directly available to agent $i$. The agents communicate the sufficient statistics needed for their model updates through the network.

Because the true labels are bounded such that $y_{t,i}\in[-Y,Y]$ for all $t\in[T]$ and $i\in[N]$, we define the prediction for a given feature $\mathbf{h}$ and model $\bx$ as the clipped inner product $\Pi_{[-Y,Y]}(\langle\mathbf h, \bx\rangle)$, where $\Pi_{[-Y,Y]}(a)\triangleq \max\{-Y,\min\{a,Y\}\}$. Furthermore, let $\widehat y_{t,j}^{i} \triangleq \Pi_{[-Y,Y]} \left( \left\langle \mathbf h_{t,j}, \mathbf x_i(t) \right\rangle \right)$ denote the prediction of agent $i$'s model evaluated on agent $j$'s feature at round $t$. For $j\ne i$, this is a virtual prediction used only to define regret; its evaluation does not require agent $i$ to observe $\mathbf h_{t,j}$ online. Clipping does not increase the squared loss because the true labels lie in $[-Y,Y]$. It also makes that loss $2Y$-Lipschitz in the unbounded raw prediction, which is essential for controlling the gossip error without imposing a bounded model domain.

We then measure the performance of each agent $i$ using the regret evaluated on the squared loss of their predictions across all global features. Specifically, the regret of agent $i$ with respect to a fixed comparator $\bx$ is defined as follows:
\begin{align*}
    \Reg_{T,i}^{\mathrm{OLR}}(\bx)
\triangleq
\sum_{t=1}^{T}
\sum_{j=1}^{N}
\frac12
\left(
\Pi_{[-Y,Y]}
\left(
\left\langle
\mathbf h_{t,j},
\mathbf x_i(t)
\right\rangle
\right)
-y_{t,j}
\right)^2
-
\sum_{t=1}^{T}\sum_{j=1}^{N}
\frac12
\left(
\left\langle
\mathbf h_{t,j},
\mathbf{x}
\right\rangle-y_{t,j}
\right)^2~.
\end{align*}
We also define the optimal baseline as $\Reg_{T,i}^{\mathrm{OLR}} \triangleq \max_{\bx \in \mathbb{R}^d}\Reg_{T,i}^{\mathrm{OLR}}(\bx)$. Notably, although squared loss is exp-concave on bounded prediction ranges, the unbounded nature of the decision domain in this setting prevents a direct application of the \pref{alg:dons} discussed in the previous section.  Before presenting our algorithm, we formalize a standard boundedness assumption regarding the feature vectors.

\begin{assumption}
    \label{asm:lr}
    The feature vectors $\{\mathbf{h}_{t,i}\}_{t \in[T], i \in N}$ have bounded norms. Specifically, there exists an $H\geq 0$ such that $\|\mathbf{h}_{t,i}\|_2\leq H$ holds for any $t \in [T]$ and $i \in [N]$.
\end{assumption}

For each block $z$, agent $i$ forms the covariance and label-feature statistics
\[
C_i^0(z)=\sum_{\tau\in\mathcal T_z}\mathbf h_{\tau,i}\mathbf h_{\tau,i}^{\top},
\qquad
Y_i^0(z)=\sum_{\tau\in\mathcal T_z}y_{\tau,i}\mathbf h_{\tau,i}.
\]
Both statistics are gossiped during block $z+1$ and are available for use in block $z+2$. We use two ridge-regression copies, assigned to odd and even blocks, so that each copy's previous statistics have finished gossiping before its next prediction~\citep{joulani2013online}. Write $\mathcal S_z=\{s:1\le s\le z-2,\ s\equiv z\pmod 2\}$ for the earlier blocks assigned to the same copy; this set is empty for $z\le2$. We use $\mathcal T_z=\{(z-1)L+1,\ldots,\min\{zL,T\}\}$ to include a final incomplete block.

\begin{algorithm}[t]
\caption{Decentralized online ridge regression for agent $i\in[N]$}
\label{alg:distributed-ridge}
\nl \textbf{Input:} gossip parameter $\theta$, epoch length $L$, regularization parameter $\lambda>0$, communication matrix $P$, clipping level $Y$\;
\nl \For{$z=1,\ldots,\lceil T/L\rceil$}{
    \nl Set $k_Y=0$\;
    \nl \For{$t=(z-1)L+1,\ldots,\min\{zL,T\}$}{
        \nl Observe the local feature $\mathbf h_{t,i}$ and compute
        \[
        \mathbf x_i(z,t)=\left(\lambda I_d+N\sum_{s\in\mathcal S_z}\Pi_{\mathbb S_+^d}(C_i(s))\right)^{-1}N\sum_{s\in\mathcal S_z}Y_i(s).
        \]\label{line: update}
        \nl Predict $\Pi_{[-Y,Y]}(\langle\mathbf h_{t,i},\mathbf x_i(z,t)\rangle)$ and receive only the local label $y_{t,i}$\;
        \nl \If{$z\ge2$}{
            \nl Simultaneously update
            \[
            C_i^{k_Y+1}(z-1)=(1+\theta)\sum_{j\in\mathcal N_i}P_{ij}C_j^{k_Y}(z-1)-\theta C_i^{k_Y-1}(z-1),
            \]
            \[
            Y_i^{k_Y+1}(z-1)=(1+\theta)\sum_{j\in\mathcal N_i}P_{ij}Y_j^{k_Y}(z-1)-\theta Y_i^{k_Y-1}(z-1).
            \]
        }
        \nl Set $k_Y\leftarrow k_Y+1$\;
    }
    \nl \If{$z\ge2$ and $k_Y=L$}{
        \nl Set $C_i(z-1)=C_i^L(z-1)$ and $Y_i(z-1)=Y_i^L(z-1)$\;
    }
    \nl Set $C_i^0(z)=C_i^{-1}(z)=\sum_{\tau\in\mathcal T_z}\mathbf h_{\tau,i}\mathbf h_{\tau,i}^{\top}$\;
    \nl Set $Y_i^0(z)=Y_i^{-1}(z)=\sum_{\tau\in\mathcal T_z}y_{\tau,i}\mathbf h_{\tau,i}$\;
}
\end{algorithm}

We now introduce \pref{alg:distributed-ridge} for the D-OLR problem. Our approach builds upon the framework of block-based accelerated gossip, integrating it with the classic online ridge regression algorithm \citep{hoerl1970ridge,vovk2001competitive}. In contrast to the general exp-concave setting over bounded domains which necessitates the communication of both cumulative gradients and decision vectors, our algorithm exploits a fundamental structural property of D-OLR: the model update depends only on covariance and label-feature statistics, both of which can be aggregated locally and communicated through gossip. Specifically, \pref{alg:distributed-ridge} proceeds over these blocks as follows:

\begin{itemize}[leftmargin=*,nosep]
    \item \textbf{Step 1: Local Ridge Prediction.}
    At each round $t\in\mathcal{T}_z$, agent $i$ computes a ridge-regression estimate $\mathbf{x}_i(z,t)$ (\pref{line: update}) using the covariance and label-feature statistics aggregated via gossip from the earlier blocks $s\in\mathcal S_z$. The model is fixed during the block. Here $\Pi_{\mathbb S_+^d}$ denotes projection onto the positive semidefinite cone in Frobenius norm. This projection is needed because accelerated gossip uses signed combinations and need not preserve positive semidefiniteness; it ensures that the regularized matrix is bounded below by $\lambda I_d$.

    \item \textbf{Step 2: Prediction and Local Feedback.}
    Then, agent $i$ deploys the decision model $\bx_i(z,t)$ and observes only its local ground-truth label $y_{t,i}$.
    With a slight abuse of notation, we define $\wh{y}_{t,j}^{i}(z,t)\triangleq\Pi_{[-Y,Y]}\left(\inner{\mathbf{h}_{t,j},\bx_i(z,t)}\right)$ as the virtual evaluation of agent $i$'s model on agent $j$'s feature using \pref{alg:distributed-ridge}.

    \item \textbf{Step 3: Gossip of Covariance and Label-Feature Statistics.}
    While collecting new local labels during block $z$, the agents simultaneously perform accelerated gossip on both statistics compiled during the previous block $z-1$, applying the covariance update entrywise.
    After $L$ gossip iterations, the local vector $Y_i^L(z-1)$ approximates the network average of the cumulative label-feature vectors $
        \frac{1}{N}
        \sum_{\tau\in\mathcal T_{z-1}}
        \sum_{j=1}^N
        y_{\tau,j}\mathbf h_{\tau,j}
    $ generated in block $z-1$. Likewise, $C_i^L(z-1)$ approximates $\frac1N\sum_{\tau\in\mathcal T_{z-1}}\sum_{j=1}^N\mathbf h_{\tau,j}\mathbf h_{\tau,j}^{\top}$. Both updates use the same neighbor-communication step, with $O(d^2)$ real numbers per message.
\end{itemize}

We remark again that in contrast to \pref{alg:dons} for the general exp-concave setting which requires gossiping both the cumulative gradients and the decision vectors, \pref{alg:distributed-ridge} requires gossiping covariance and label-feature statistics, but not the decision vectors. The absence of a decision-gossip stage allows the procedure to operate with only a 1-block delay, rather than the 2-block delay inherent to \pref{alg:dons} as discussed before. Specifically, after agents collect their local covariance and label-feature statistics in block $z$, they execute the two gossip recursions simultaneously during block $z+1$ to approximate the network averages. These statistics are then incorporated into the same ridge-regression copy for use in block $z+2$. Thus, each copy uses all of its completed earlier blocks, and the interleaving costs only an absolute factor in regret.

The following theorem demonstrates that \pref{alg:distributed-ridge} guarantees a logarithmic regret for D-OLR.

\begin{samepage}
\begin{restatable}{theorem}{theoremlr}
\label{thm: lr}
Suppose \pref{asm:lr} holds, and let $P$ be the communication matrix defined in \pref{def:5}. In D-OLR, if every agent $i\in[N]$ executes \pref{alg:distributed-ridge} with $\lambda>0$, $\theta$ chosen according to \pref{eqn:para-choice}, and
\begin{equation}\label{eq:lr-block-length}
L=\left\lceil\frac{\sqrt2}{(\sqrt2-1)\sqrt{1-\sigma_2(P)}}\log\!\left(\sqrt{14}\,T^3N^{7/2}(1+H^2/\lambda)\right)\right\rceil,
\end{equation}
then for any agent $i\in[N]$ and comparator $\mathbf x\in\mathbb R^d$, the regret is bounded by:
\begin{align}
\label{eq:olrbound}
   \Reg_{T,i}^{\mathrm{OLR}}(\mathbf{x}) = \mathcal{O} \left(\lambda \|\mathbf x\|_2^2
+ \frac{Y^2  H^2}{\lambda} +
Y^2LN d
\log\left(
1+\frac{NTH^2}{\lambda d}
\right)\right)~.
\end{align}
\end{restatable}
\end{samepage}

By substituting the choice of $L$ from \pref{eq:lr-block-length} into \pref{eq:olrbound}, we obtain a total regret of $\mathcal{O}\left(Nd (1-\sigma_2(P))^{-1/2}\,\log^2(T)\right)$. To our knowledge, \pref{thm: lr} establishes the first logarithmic regret guarantee for D-OLR with locally observed features and labels.

\subsection{Analysis of~\pref{alg:distributed-ridge}}
We briefly explain the main ideas behind the proof of \pref{thm: lr}; the complete proof is
deferred to \pref{app:appdolr}. The analysis compares \pref{alg:distributed-ridge} with a centralized online ridge-regression predictor that has access to all labels immediately. We make this comparison separately for the odd and even copies. Write
\[
A(z)=\lambda I_d+\sum_{s\in\mathcal S_z}\sum_{j=1}^N C_j^0(s),
\qquad b(z)=\sum_{s\in\mathcal S_z}\sum_{j=1}^N Y_j^0(s)
\]
for the exact delayed covariance and label-feature statistics. For each block \(z\) and round \(t\in\mathcal T_z\), define the reference predictor
\[
\mathbf x^\star(z,t)=\left(A(z)+\sum_{\substack{\tau\in\mathcal T_z\\\tau\le t}}\sum_{j=1}^N\mathbf h_{\tau,j}\mathbf h_{\tau,j}^{\top}\right)^{-1}\left(b(z)+\sum_{\substack{\tau\in\mathcal T_z\\\tau\le t}}\sum_{j=1}^N y_{\tau,j}\mathbf h_{\tau,j}\right),
\]
which corresponds to the centralized ridge-regression solution using all observations up to round \(t\) assigned to the same copy. It is used only for the analysis. The regret is then decomposed into two parts:
\begin{align}\label{eqn:olr-de}
\mathrm{Reg}^{\mathrm{OLR}}_{T,i}(\mathbf x)
&=\sum_{z=1}^{\lceil T/L\rceil}\sum_{t\in\mathcal T_z}\sum_{j=1}^N
\bigl(f_{t,j}(\mathbf x^\star(z,t))-f_{t,j}(\mathbf x)\bigr)\nonumber\\
&\quad+\sum_{z=1}^{\lceil T/L\rceil}\sum_{t\in\mathcal T_z}\sum_{j=1}^N
\bigl(\widehat f^i_{t,j}(z,t)-f_{t,j}(\mathbf x^\star(z,t))\bigr),
\end{align}
where we define $f_{t,j}(\mathbf x)=\frac12(\langle\mathbf h_{t,j},\mathbf x\rangle-y_{t,j})^2$ and
$\widehat f^i_{t,j}(z,t)\triangleq\frac12(\widehat y^i_{t,j}(z,t)-y_{t,j})^2$.
The first term is precisely the sum of the regrets of the two centralized ridge-regression reference sequences. Applying the Be-The-Regularized-Leader lemma to each copy bounds this term by $\lambda\|\mathbf x\|_2^2$.

It remains to control the discrepancy between the decentralized predictor and the full-information ridge predictor. As the clipping operation guarantees that the squared loss is $2Y$-Lipschitz with respect to the raw predictions, the loss difference between two clipped predictors is at most $2Y$ times their prediction difference. This is why clipping is useful here: the models themselves can be unbounded, whereas the error estimate depends only on the bounded label range. Moreover, clipping the reference prediction cannot increase its squared loss. The complete scalar argument is given in \pref{app:appdolr}.

We introduce an auxiliary exact delayed ridge predictor $\mathbf x'(z,t)=A(z)^{-1}b(z)$. The analysis then separates a \textit{delay error}, due to the absence of the current block's covariance and label-feature statistics, from a \textit{gossip error}, due to imperfect averaging of completed blocks.

For the delay error, each copy is ordinary ridge regression over batches of at most $LN$ examples. We retain the regularized-minimum increments in the centralized comparison and use the quadratic structure to control the loss of the fixed model throughout a batch. Clipping bounds each batch loss by $2Y^2LN$. A matrix-inversion identity then bounds the excess over the regularized-minimum increment by $2Y^2LN$ times the increase in a log-determinant potential. Summing these increases over the odd and even copies gives
$\mathcal O\bigl(Y^2LNd\log(1+NTH^2/(\lambda d))\bigr)$.
This accounts jointly for the delay of the covariance and label-feature statistics; the intermediate calculations are included in the full proof.

For the gossip error, the covariance estimates are projected onto the positive semidefinite cone, so both the exact and approximate regularized matrices dominate $\lambda I_d$. Non-expansiveness of this projection and the matrix-inverse identity
$\widehat A^{-1}-A^{-1}=\widehat A^{-1}(A-\widehat A)A^{-1}$
reduce the model difference to the covariance and label-feature gossip errors. Using the same accelerated-gossip estimate as in \pref{lem:gossip-close-exp}, with the block length in \pref{eq:lr-block-length}, their cumulative contribution to regret is at most $4Y^2H^2/\lambda$. Summing the centralized, delay, and gossip terms yields the claimed regret bound.

\bibliographystyle{abbrvnat}
\bibliography{ref}
\clearpage
\appendix
\section{Auxiliary Lemmas}
\label{sec: app_aux}
In this section, we provide some auxiliary technical results that will be helpful throughout the paper.

The following lemma is a standard result standard be-the-leader lemma, see also \citep[Theorem~3]{gyorgy2021adapting}.
\begin{lemma}[Be-The-Regularized-Leader]
\label{lem:btrl}
Let $\ell_1,\dots,\ell_T:\cX\to\mathbb{R}$ be arbitrary functions, and let
$\psi_0,\psi_1,\dots,\psi_T:\cX\to\mathbb{R}$ satisfy $\psi_0\equiv 0$.
Define
$
F_t(\bx)
\triangleq
\sum_{s=1}^t \ell_s(\bx)+\psi_t(\bx)
$
and
$
\bx^\star(t)\in\arg\min_{\bx\in\cX}F_t(\bx),
\quad t\in[T].
$
Then, for any $\bx\in\cX$,
\begin{align}
\sum_{t=1}^T \ell_t\bigl(\bx^\star(t)\bigr)
\le
\sum_{t=1}^T \ell_t(\bx)+\psi_T(\bx)
-\sum_{t=1}^T\Bigl(\psi_t\bigl(\bx^\star(t)\bigr)-\psi_{t-1}\bigl(\bx^\star(t)\bigr)\Bigr).
\label{eq:btrl}
\end{align}
In particular, if
\[
\psi_t(\bx)-\psi_{t-1}(\bx)\ge 0,
\qquad \forall t\in[T],\ \forall \bx\in\cX,
\]
then
\[
\sum_{t=1}^T \ell_t\bigl(\bx^\star(t)\bigr)
\le
\sum_{t=1}^T \ell_t(\bx)+\psi_T(\bx).
\]
\end{lemma}

\begin{proof}
Define
\[
\tilde{\ell}_t(\bx)
\triangleq
\ell_t(\bx)+\psi_t(\bx)-\psi_{t-1}(\bx),
\qquad t\in[T].
\]
Then, for every $t\in[T]$,
\[
\sum_{s=1}^t \tilde{\ell}_s(\bx)
=
\sum_{s=1}^t \ell_s(\bx)
+\sum_{s=1}^t\bigl(\psi_s(\bx)-\psi_{s-1}(\bx)\bigr)
=
\sum_{s=1}^t \ell_s(\bx)+\psi_t(\bx)
=
F_t(\bx).
\]
Hence,
\[
\bx^\star(t)\in\arg\min_{\bx\in\cX}\sum_{s=1}^t \tilde{\ell}_s(\bx).
\]

Applying the Be-The-Leader lemma to the sequence
$\{\tilde{\ell}_t\}_{t=1}^T$, we obtain that for any $\bx\in\cX$,
\[
\sum_{t=1}^T \tilde{\ell}_t\bigl(\bx^\star(t)\bigr)
\le
\sum_{t=1}^T \tilde{\ell}_t(\bx).
\]
Expanding both sides yields
\[
\sum_{t=1}^T
\Bigl(
\ell_t\bigl(\bx^\star(t)\bigr)
+\psi_t\bigl(\bx^\star(t)\bigr)
-\psi_{t-1}\bigl(\bx^\star(t)\bigr)
\Bigr)
\le
\sum_{t=1}^T
\Bigl(
\ell_t(\bx)+\psi_t(\bx)-\psi_{t-1}(\bx)
\Bigr).
\]
Since the right-hand side telescopes and $\psi_0\equiv 0$, we have
\[
\sum_{t=1}^T
\bigl(\psi_t(\bx)-\psi_{t-1}(\bx)\bigr)
=
\psi_T(\bx).
\]
Therefore,
\[
\sum_{t=1}^T \ell_t\bigl(\bx^\star(t)\bigr)
+
\sum_{t=1}^T
\Bigl(
\psi_t\bigl(\bx^\star(t)\bigr)-\psi_{t-1}\bigl(\bx^\star(t)\bigr)
\Bigr)
\le
\sum_{t=1}^T \ell_t(\bx)+\psi_T(\bx).
\]
Rearranging proves \eqref{eq:btrl}. If, in addition,
$
\psi_t(\bx)-\psi_{t-1}(\bx)\ge 0$, $\forall t\in[T],\ \forall \bx\in\cX,
$
then
\[
\sum_{t=1}^T
\Bigl(
\psi_t\bigl(\bx^\star(t)\bigr)-\psi_{t-1}\bigl(\bx^\star(t)\bigr)
\Bigr)\ge 0.
\]
Dropping this nonnegative term gives
\[
\sum_{t=1}^T \ell_t\bigl(\bx^\star(t)\bigr)
\le
\sum_{t=1}^T \ell_t(\bx)+\psi_T(\bx),
\]
which completes the proof.
\end{proof}

The following lemma is the quadratic bound of $\alpha$-exp-concave functions.
\begin{lemma}[{\citep[Lemma~3]{hazan2007logarithmic}}]
\label{lem: expconcave}
At each round $t\in[T]$, the loss function $f_{t,i}(\x)$ of each agent $i\in [N]$ is $\alpha$-exp-concave over $\calX$, i.e., $f_{t,i}(\mathbf{y})\geq f_{t,i}(\x)+\langle \nabla f_{t,i}(\x),\mathbf{y}-\x\rangle
+\frac{\beta}{2}\big(\langle \nabla f_{t,i}(\x),\mathbf{y}-\x\rangle\big)^2$ for any $\x,\mathbf{y}\in \calX$, where $\beta = \frac12\min \left\{\frac{1}{4GD}, \alpha\right\}$.
\end{lemma}

Here we provide a proof of \pref{lem:gossip-close-exp}, which shows that after $\Theta((1-\sigma_2)^{-1/2}\log(TN))$ rounds of accelerated gossiping, the approximation error to the global average is $\order(1/(T^2N^2))$.

\gossip*
\begin{proof}
By Proposition~1 of
\citet{ye2023multi}, the accelerated gossip iterations satisfy\footnote{In the proof of Proposition~1 of \citet{ye2023multi}, the upper-right block of the transition matrix should be $-\eta_w I$, rather than $-\eta_w W$. With this correction, the projection identities still hold, so the proof requires only this minor adjustment and the proposition remains valid.}
\begin{equation*}
    \|U^L-\overline U\|_F
    \le
    \sqrt{14}
    \left(
    1-\left(1-\frac{1}{\sqrt{2}}\right)
    \sqrt{1-\sigma_2(P)}
    \right)^L
    \|U^0-\overline U\|_F .
\end{equation*}
By the definition of $L$ and the inequality $\log(1-x)\le -x$ for $x\in(0,1)$, with
$a=\left(1-\frac{1}{\sqrt{2}}\right)\sqrt{1-\sigma_2(P)}$, we have
$L\ge a^{-1}\log(T^2\sqrt{14}N^{5/2})$, and hence
\begin{equation*}
    \sqrt{14}
    \left(
    1-\left(1-\frac{1}{\sqrt{2}}\right)
    \sqrt{1-\sigma_2(P)}
    \right)^L
    =
    \sqrt{14}(1-a)^L
    \le
    \sqrt{14}e^{-aL}
    \le
    \frac{1}{T^2N^{5/2}}.
\end{equation*}
Moreover, since $\|\mathbf{u}_i\|_2\le R$ for all $i\in[N]$, we also have
\[
    \|\overline{\mathbf{u}}\|_2
    =
    \left\|
    \frac{1}{N}\sum_{i=1}^N \mathbf{u}_i
    \right\|_2
    \le
    \frac{1}{N}\sum_{i=1}^N \|\mathbf{u}_i\|_2
    \le
    R .
\]
Hence, for every $i\in[N]$,
\[
    \|\mathbf{u}_i^0-\overline{\mathbf{u}}\|_2
\le
\|\mathbf{u}_i^0\|_2+\|\overline{\mathbf{u}}\|_2
=
\|\mathbf{u}_i\|_2+\|\overline{\mathbf{u}}\|_2
\le
2R .
\]
Therefore,
\[
    \|U^0-\overline U\|_F
    =
    \left(
    \sum_{i=1}^N
    \|\mathbf{u}_i^0-\overline{\mathbf{u}}\|_2^2
    \right)^{1/2}
    \le
    2R\sqrt N.
\]
Combining the previous inequalities gives
\[
    \|U^L-\overline U\|_F
    \le
    \frac{1}{T^2 N^{5/2}}\cdot 2R\sqrt N
    = 
    \frac{2R}{T^2 N^2}.
\]
Since the Euclidean norm of each row is upper bounded by the Frobenius norm,
we obtain
\begin{equation}
    \|\mathbf{u}_i^L-\overline{\mathbf{u}}\|_2
    \le
    \|U^L-\overline U\|_F
    \le
    \frac{2R}{T^2N^2},
    \qquad i\in[N].
\end{equation}
This completes the proof.
\end{proof}
\section{Omitted Details from \pref{sec:d-exp-concave}}
\label{app:appexpconcave}

In this section, we provide the omitted details from \pref{sec:d-exp-concave}.
We begin by introducing two key lemmas for the proof of \pref{thm:upper-exp}. We then combine the two lemmas to prove
\pref{thm:upper-exp}, and finally prove the lower bound in \pref{thm:lowerbound}.

\pref{lem:regret-decomposition} reduces the decentralized regret of any agent to a sum of block-level surrogate losses. It shows that, after accelerated gossip, the regret can be controlled by linearized gradients plus an exp-concave quadratic correction term, up to lower-order consensus errors. The key role of this lemma is to connect the decentralized online exp-concave problem with the delayed ONS analysis used later.

Compared with \citet[Theorem~1]{wan2025black}, which uses the uniform quadratic curvature of strongly convex losses, the reduction below retains the gradient-dependent rank-one correction of exp-concavity. The averaging and Jensen steps are standard; their role here is to preserve this correction in the expression passed to \pref{lem:block-ons-exp}.

\begin{restatable}{lemma}{regretdecomposition}
\label{lem:regret-decomposition}
Under the same conditions as \pref{alg:dons}. For any comparator \(\mathbf{x}\in\mathcal X\), the regret of any agent \(i\in[N]\) satisfies
\[
\Reg_{T,i}(\mathbf{x})
\le
\sum_{z=1}^{T/L}\sum_{j=1}^N
\left(
    \left\langle \mathbf{g}_j(z),\widehat{\mathbf{x}}_j(z)-\mathbf{x}\right\rangle
    -
    \frac{\beta}{2L}
    \left\langle \mathbf{g}_j(z),\mathbf{x}_j(z)-\mathbf{x}\right\rangle^2
\right)
+
\mathcal{O}\left(\frac{GD}{NT} + \frac{\beta G^2 D^2}{NT}\right).
\]
\end{restatable}
\begin{proof}
Before proving the lemma, we first record two consensus bounds implied by
\pref{lem:gossip-close-exp}. For each block \(z\in [T/L]\), recall
\[
\overline{\mathbf{x}}(z)
=
\frac{1}{N}\sum_{j=1}^N \widehat{\mathbf{x}}_j(z),
\qquad
\overline{\mathbf{g}}(z)
=
\frac{1}{N}
\sum_{t\in\mathcal T_z}\sum_{j=1}^N
\nabla f_{t,j}\left(\mathbf{x}_j(z)\right).
\]

By \pref{asm: domain}, the set \(\mathcal X\) contains
the origin and has diameter at most \(D\). Since
\(\widehat{\mathbf{x}}_i(z)\in \mathcal X\), we have
\(\|\widehat{\mathbf{x}}_i(z)\|_2\le D\). Hence, after \(L\) accelerated gossip
steps in \pref{alg:dons}, and using the non-expansiveness of projection onto a
convex set \citep[Proposition~2.11]{orabona2026}, the played decision satisfies
\begin{equation}
\label{eq:decision-consensus}
\left\|
\mathbf{x}_i(z)-\overline{\mathbf{x}}(z)
\right\|_2
\le
\left\|
\mathbf{x}_i^L(z)-\overline{\mathbf{x}}(z)
\right\|_2
\le
\frac{2D}{T^2N^2}.
\end{equation}

Similarly, by \pref{asm: cvx}, each local loss is \(G\)-Lipschitz over
\(\mathcal X\), and hence
\[
\left\|
\sum_{t\in T_z}\nabla f_{t,i}(\mathbf{x}_i(z))
\right\|_2
\le LG.
\]
Applying \pref{lem:gossip-close-exp} to the gradient-gossip process gives
\begin{equation}
\label{eq:gradient-consensus}
\left\|
\mathbf{g}_i(z)- \overline{\mathbf{g}}(z)
\right\|_2
\le
\frac{2LG}{T^2N^2}.
\end{equation}

For any $\mathbf{x} \in \calX$, by \pref{asm: cvx}, we have 
\begin{align}
f_{t, j}\left(\mathbf{x}_i(z)\right)-f_{t, j}(\mathbf{x}) 
&\leq  f_{t, j}\left(\mathbf{x}_j(z)\right)-f_{t, j}(\mathbf{x})+G\left\|\mathbf{x}_j(z)-\mathbf{x}_i(z)\right\|_2   \notag
\\ &\leq  \left\langle\nabla f_{t, j}\left(\mathbf{x}_j(z)\right), \mathbf{x}_j(z)-\mathbf{x}\right\rangle-\frac{\beta}{2}\inner{\nabla f_{t, j}\left(\mathbf{x}_j(z)\right), \mathbf{x}_j(z)-\mathbf{x}}^2 \notag
\\&\quad+G\left\|\mathbf{x}_j(z)-\overline{\mathbf{x}}(z)+\overline{\mathbf{x}}(z)-\mathbf{x}_i(z)\right\|_2 \tag{\pref{asm: expconcave} and \pref{lem: expconcave} }
\\
&\leq  \left\langle\nabla f_{t, j}\left(\mathbf{x}_j(z)\right), \overline{\mathbf{x}}(z)-\mathbf{x}\right\rangle-\frac{\beta}{2}\inner{\nabla f_{t, j}\left(\mathbf{x}_j(z)\right), \mathbf{x}_j(z)-\mathbf{x}}^2 \notag
\\& \quad+\left\langle\nabla f_{t, j}\left(\mathbf{x}_j(z)\right), \mathbf{x}_j(z)-\overline{\mathbf{x}}(z)\right\rangle+\frac{4 G D}{T^2N^2}  \notag 
\\
& \leq \left\langle\nabla f_{t, j}\left(\mathbf{x}_j(z)\right), \overline{\mathbf{x}}(z)-\mathbf{x}\right\rangle-\frac{\beta}{2}\inner{\nabla f_{t, j}\left(\mathbf{x}_j(z)\right), \mathbf{x}_j(z)-\mathbf{x}}^2 \notag
\\&\quad+G\left\|\mathbf{x}_j(z)-\overline{\mathbf{x}}(z)\right\|_2+\frac{4 G D}{T^2N^2} \notag \\
&\leq  \left\langle\nabla f_{t, j}\left(\mathbf{x}_j(z)\right), \overline{\mathbf{x}}(z)-\mathbf{x}\right\rangle-\frac{\beta}{2}\inner{\nabla f_{t, j}\left(\mathbf{x}_j(z)\right), \mathbf{x}_j(z)-\mathbf{x}}^2+\frac{6 G D}{T^2N^2}~,
\label{eq: regret-de-1}
\end{align}
where the third and the last inequalities use \pref{eq:decision-consensus}. Summing over \(z\in[T/L]\), \(t\in\mathcal T_z\), and \(j\in[N]\), we have
\begin{align}
&\sum_{z=1}^{T/L}\sum_{t \in \mathcal{T}_z} \sum_{j=1}^N\left(f_{t, j}\left(\mathbf{x}_i(z)\right)-f_{t, j}(\mathbf{x})\right) \notag
\\&\leq  \sum_{z=1}^{T/L} \sum_{t \in \mathcal{T}_z}\sum_{j=1}^N  \left( \left\langle\nabla f_{t, j}\left(\mathbf{x}_j(z)\right), \overline{\mathbf{x}}(z)-\mathbf{x}\right\rangle- \frac{\beta}{2}\inner{\nabla f_{t, j}\left(\mathbf{x}_j(z)\right), \mathbf{x}_j(z)-\mathbf{x}}^2\right)+\frac{6 G D}{T N} \tag{\pref{eq: regret-de-1}}
\\&= \sum_{z=1}^{T/L}  \left( N\left\langle \overline{\mathbf{g}}(z), \overline{\mathbf{x}}(z)-\mathbf{x}\right\rangle-\frac{\beta}{2}\sum_{t \in \mathcal{T}_z}\sum_{j=1}^N\inner{\nabla f_{t, j}\left(\mathbf{x}_j(z)\right), \mathbf{x}_j(z)-\mathbf{x}}^2\right)+\frac{6GD}{T N}\tag{$\overline{\mathbf{g}}(z)=\frac{1}{N} 
    \sum_{t\in \calT_z}\sum_{j=1}^N \nabla f_{t, j}\left(\mathbf{x}_j(z)\right)$}
\\&\leq \sum_{z=1}^{T/L}\left( N\left\langle \overline{\mathbf{g}}(z), \overline{\mathbf{x}}(z)-\mathbf{x}\right\rangle-\frac{\beta}{2NL}
\left(
\sum_{t \in \mathcal{T}_z}\sum_{j=1}^N
\inner{\nabla f_{t, j}\left(\mathbf{x}_j(z)\right), \mathbf x_j(z)-\mathbf{x}}
\right)^2\right)+\frac{6GD}{T N},
\label{eq: decom_1}
\end{align}
where the last inequality holds because of Jensen's inequality. Now we focus on 
\begin{align}
&-\frac{\beta}{2NL}
\left(
\sum_{t \in \mathcal{T}_z}\sum_{j=1}^N
\inner{\nabla f_{t, j}\left(\mathbf{x}_j(z)\right), \mathbf x_j(z)-\mathbf{x}}
\right)^2 \notag\\
&=-\frac{\beta}{2NL}
\left(
\sum_{t \in \mathcal{T}_z}\sum_{j=1}^N
\inner{\nabla f_{t, j}\left(\mathbf{x}_j(z)\right), \mathbf x_j(z)-\overline{\mathbf{x}}(z)} + \sum_{t \in \mathcal{T}_z}\sum_{j=1}^N
\inner{\nabla f_{t, j}\left(\mathbf{x}_j(z)\right), \overline{\mathbf{x}}(z)-\mathbf{x}}
\right)^2 \notag\\
&\leq -\frac{N\beta}{2L}
\inner{\overline{\bg}(z), \overline{\mathbf{x}}(z)-\mathbf{x}}^2
+\beta L G^2 D \sum_{j=1}^N
\|\mathbf x_j(z)-\overline{\mathbf x}(z)\|_2 \notag\\
&= -\frac{\beta}{2L}
\sum_{j=1}^N
\inner{\overline{\bg}(z), \overline{\mathbf{x}}(z)-\mathbf{x}}^2
+\beta L G^2 D \sum_{j=1}^N
\|\mathbf x_j(z)-\overline{\mathbf x}(z)\|_2 \notag\\
&\leq -\frac{\beta}{2L}
\sum_{j=1}^N
\inner{\bg_j(z), \overline{\mathbf{x}}(z)-\mathbf{x}}^2
+\beta L G^2 D \sum_{j=1}^N
\|\mathbf x_j(z)-\overline{\mathbf x}(z)\|_2
+\beta G D^2 \sum_{j=1}^N
\|\bg_j(z)-\overline{\bg}(z)\|_2 \notag\\
&\leq -\frac{\beta}{2L}
\sum_{j=1}^N
\inner{\bg_j(z), \mathbf{x}_j(z)-\mathbf{x}}^2
+2\beta L G^2 D \sum_{j=1}^N
\|\mathbf x_j(z)-\overline{\mathbf x}(z)\|_2
+\beta G D^2 \sum_{j=1}^N
\|\bg_j(z)-\overline{\bg}(z)\|_2 \notag\\
&\leq -\frac{\beta}{2L}
\sum_{j=1}^N
\inner{\bg_j(z), \mathbf{x}_j(z)-\mathbf{x}}^2 + \frac{6\beta G^2 D^2L}{T^2N}
\label{eq: decom_2}
\end{align}
Here, the first inequality is obtained by decomposing
$
\mathbf{x}_j(z)-\mathbf{x}=\bigl(\overline{\mathbf{x}}(z)-\mathbf{x}\bigr)+\bigl(\mathbf{x}_j(z)-\overline{\mathbf{x}}(z)\bigr),
$
and applying \( -(a+b)^2\le -a^2+2|a||b| \), Cauchy--Schwarz, \pref{asm: cvx}, and \pref{asm: domain}. The equality follows from the fact that \(\inner{\overline{\bg}(z), \overline{\mathbf{x}}(z)-\mathbf{x}}^2\) is independent of \(j\). For the second inequality, we use the identity
$
|a^2-b^2|=|a-b||a+b|
$
with
$
a=\inner{\overline{\bg}(z),\overline{\mathbf{x}}(z)-\mathbf{x}},
$ and $
b=\inner{\bg_j(z),\overline{\mathbf{x}}(z)-\mathbf{x}}
$
,together with Cauchy--Schwarz, \pref{asm: cvx}, and \pref{asm: domain}. The third inequality is proved analogously by taking
$
a=\inner{\bg_j(z),\overline{\mathbf{x}}(z)-\mathbf{x}}
$ and $
b=\inner{\bg_j(z),\mathbf{x}_j(z)-\mathbf{x}}
$.The last inequality follows from \pref{eq:decision-consensus} and \pref{eq:gradient-consensus}.

Combining \pref{eq: decom_1} and \pref{eq: decom_2}, we have 
\begin{align*}
    &\sum_{z=1}^{T/L}\sum_{t \in \mathcal{T}_z} \sum_{j=1}^N\left(f_{t, j}\left(\mathbf{x}_i(z)\right)-f_{t, j}(\mathbf{x})\right)
    \\&\leq \sum_{z=1}^{T/L}\left( N\left\langle \overline{\mathbf{g}}(z), \overline{\mathbf{x}}(z)-\mathbf{x}\right\rangle-\frac{\beta}{2L}
\sum_{j=1}^N
\inner{\bg_j(z), \mathbf{x}_j(z)-\mathbf{x}}^2\right)+\frac{6 G D}{TN} + \frac{6\beta G^2 D^2}{TN}
\\&= \sum_{z=1}^{T/L} \sum_{j=1}^N \left( \left\langle \overline{\mathbf{g}}(z), \hat{\mathbf{x}}_j(z)-\mathbf{x}\right\rangle-\frac{\beta}{2L}
\inner{\bg_j(z), \mathbf{x}_j(z)-\mathbf{x}}^2\right)+\frac{6 G D}{TN} + \frac{6\beta G^2 D^2}{TN} \tag{$ \overline{\mathbf{x}}(z) = \frac{1}{N}\sum_{j=1}^N \hat{\mathbf{x}}_j(z)$}
\\&\leq \sum_{z=1}^{T/L} \sum_{j=1}^N \left( \left\langle \mathbf{g}_j(z), \hat{\mathbf{x}}_j(z)-\mathbf{x}\right\rangle-\frac{\beta}{2L}
\inner{\bg_j(z), \mathbf{x}_j(z)-\mathbf{x}}^2\right)+\frac{6 G D}{TN} + \frac{6\beta G^2 D^2}{TN} \\& \quad+ D\sum_{z=1}^{T/L}\sum_{j=1}^N\|\bg_j(z)- \overline{\bg}(z)\|_2
\\&= \sum_{z=1}^{T/L} \sum_{j=1}^N \left( \left\langle \mathbf{g}_j(z), \hat{\mathbf{x}}_j(z)-\mathbf{x}\right\rangle-\frac{\beta}{2L}
\inner{\bg_j(z), \mathbf{x}_j(z)-\mathbf{x}}^2\right)+ \mathcal{O}\left(\frac{GD}{NT} + \frac{\beta G^2 D^2}{NT}\right),
\end{align*}
where the last inequality follows from \pref{eq:gradient-consensus}.
\end{proof}

It remains to control the block-level surrogate regrets appearing in
\pref{lem:regret-decomposition}. This step is more delicate than in the
strongly convex case because the linear term and the quadratic correction are
evaluated at two different points: the preparatory decision
\(\widehat{\mathbf{x}}_j(z)\) and the executed decision \(\mathbf{x}_j(z)\).
To handle this mismatch, we compare the delayed FTRL iterate
\(\widehat{\mathbf{x}}_j(z)\), which only uses feedback up to block \(z-3\), with
a current regularized leader that uses feedback up to block \(z\). The comparison
is carried out in the adaptive second-order norm induced by the accumulated
gradient matrix.

\lemdelayedons*
\begin{proof}
For all $j \in [N]$, define
$
F_{z,j}^\star(\mathbf{x})
\triangleq
\sum_{s=1}^{z}\langle \bg_j(s),\mathbf{x}\rangle+\psi_{z,j}^\star(\mathbf{x})$
and let $
\mathbf{x}_j^\star(z)\in\arg\min_{\mathbf{x}\in\cX}F_{z,j}^\star(\mathbf{x}),
$
where
$
\psi_{z,j}^\star(\mathbf{x})
\triangleq
\frac{\beta}{2L}\sum_{s=1}^{z}\langle \bg_j(s),\mathbf{x}_j(s)-\mathbf{x}\rangle^2
+\frac{\eta}{2}\|\mathbf{x}\|_2^2~.
$
For all \(j\in[N]\), we have
\begin{align}
\label{eq:delayons_dec}
\sum_{z=1}^{T/L}\langle \bg_j(z), \hat{\mathbf{x}}_j(z)-\mathbf{x}\rangle
&=
\sum_{z=1}^{T/L}\langle \bg_j(z), \hat{\mathbf{x}}_j(z)-\mathbf{x}_j^\star(z)\rangle
+
\sum_{z=1}^{T/L}\langle \bg_j(z), \mathbf{x}_j^\star(z)-\mathbf{x}\rangle~.
\end{align}
Using Lemma~\ref{lem:btrl}, 
we have
\begin{align}
\label{eq: delayons_1}
\sum_{z=1}^{T/L} \langle \bg_j(z), \mathbf{x}_j^\star(z)-\mathbf{x}\rangle
&\leq
\psi_{T/L,j}^\star(\mathbf{x})
-
\sum_{z=1}^{T/L}
\Bigl(
\psi_{z,j}^\star\bigl(\mathbf{x}_j^\star(z)\bigr)
-
\psi_{z-1,j}^\star\bigl(\mathbf{x}_j^\star(z)\bigr)
\Bigr) \leq
\psi_{T/L,j}^\star(\mathbf{x}),
\end{align}
where the last inequality follows from
$
\psi_{z,j}^\star(\mathbf{x})-\psi_{z-1,j}^\star(\mathbf{x})
=
\frac{\beta}{2L}\langle \bg_j(z),\mathbf{x}_j(z)-\mathbf{x}\rangle^2
\ge 0
$.

It remains to bound
\[
\sum_{z=1}^{T/L}\langle \bg_j(z), \hat{\mathbf{x}}_j(z)-\mathbf{x}_j^\star(z)\rangle.
\]

For notational convenience, for any integer \(z\le 0\), we set
$
A_j(z)\triangleq \eta I_d,
$
For any integer \(z\ge 1\), define
\[
A_j(z)\triangleq\eta I_d+\frac{\beta}{L}\sum_{s=1}^{z}\bg_j(s)\bg_j(s)^\top.
\]
For each \(z\ge 1\), We define
\begin{align*}
\hat{F}_{z,j}(\mathbf{x})
&\triangleq
\sum_{s=1}^{z-3}
\left(
\langle \bg_j(s),\mathbf{x}\rangle
+
\frac{\beta}{2L}\langle \bg_j(s),\mathbf{x}_j(s)-\mathbf{x}\rangle^2
\right)
+\frac{\eta}{2}\|\mathbf{x}\|_2^2
\end{align*}
and $
\hat{\mathbf{x}}_j(z)\in\arg\min_{\mathbf{x}\in\cX}\hat{F}_{z,j}(\mathbf{x})$.
Since \(F_{z,j}^\star\) is \(1\)-strongly convex with respect to
\(\|\cdot\|_{A_j(z)}\), and \(\mathbf{x}_j^\star(z)\) minimizes \(F_{z,j}^\star\) over
\(\cX\), we have
\begin{align}
F_{z,j}^\star(\hat{\mathbf{x}}_j(z)) - F_{z,j}^\star(\mathbf{x}_j^\star(z))
\ge
\frac12\|\hat{\mathbf{x}}_j(z)-\mathbf{x}_j^\star(z)\|_{A_j(z)}^2.
\label{eq:block-comp-1}
\end{align}
Likewise, since \(\hat{F}_{z,j}\) is \(1\)-strongly convex with respect to
\(\|\cdot\|_{A_{j}(z-3)}\), and \(\hat{\mathbf{x}}_j(z)\) minimizes \(\hat{F}_{z,j}\) over
\(\cX\), we have
\begin{align}
\hat{F}_{z,j}(\mathbf{x}_j^\star(z))-\hat{F}_{z,j}(\hat{\mathbf{x}}_j(z))
\ge
\frac12\|\hat{\mathbf{x}}_j(z)-\mathbf{x}_j^\star(z)\|_{A_j(z-3)}^2.
\label{eq:block-comp-2}
\end{align}
Adding \pref{eq:block-comp-1} and \pref{eq:block-comp-2} is equivalent to applying the stability lemma of \citet[Lemma~A.2]{qiu2025exploiting} to these two regularized objectives. This is a standard strong-convexity argument, not a new ingredient; we include the derivation to make the treatment of the two decisions explicit.

Summing \pref{eq:block-comp-1} and \pref{eq:block-comp-2}, and expanding the
definitions of \(F_{z,j}^\star\) and \(\hat{F}_{z,j}\), yields
\begin{align*}
&\frac12\|\hat{\mathbf{x}}_j(z)-\mathbf{x}_j^\star(z)\|_{A_j(z)}^2
+
\frac12\|\hat{\mathbf{x}}_j(z)-\mathbf{x}_j^\star(z)\|_{A_{j}(z-3)}^2 \\
&\le
\sum_{s=z-2}^{z}
\langle \bg_j(s),\hat{\mathbf{x}}_j(z)-\mathbf{x}_j^\star(z)\rangle \\
&\quad
+
\frac{\beta}{2L}
\sum_{s=z-2}^{z}
\Bigl(
\langle \bg_j(s),\mathbf{x}_j(s)-\hat{\mathbf{x}}_j(z)\rangle^2
-
\langle \bg_j(s),\mathbf{x}_j(s)-\mathbf{x}_j^\star(z)\rangle^2
\Bigr).
\end{align*}
Using \(a^2-b^2=(a-b)(a+b)\), with
$
a=\langle \bg_j(s),\mathbf{x}_j(s)-\hat{\mathbf{x}}_j(z)\rangle
$ and $ 
b=\langle \bg_j(s),\mathbf{x}_j(s)-\mathbf{x}_j^\star(z)\rangle
$
we obtain
\begin{align*}
&\frac12\|\hat{\mathbf{x}}_j(z)-\mathbf{x}_j^\star(z)\|_{A_j(z)}^2
+
\frac12\|\hat{\mathbf{x}}_j(z)-\mathbf{x}_j^\star(z)\|_{A_{j}(z-3)}^2 \\
&\le
\sum_{s=z-2}^{z}
\bigl|\langle \bg_j(s),\hat{\mathbf{x}}_j(z)-\mathbf{x}_j^\star(z)\rangle\bigr| \\
&\qquad
+
\frac{\beta}{2L}
\sum_{s=z-2}^{z}
\bigl|\langle \bg_j(s),\hat{\mathbf{x}}_j(z)-\mathbf{x}_j^\star(z)\rangle\bigr|
\cdot
\bigl|\langle \bg_j(s),\hat{\mathbf{x}}_j(z)+\mathbf{x}_j^\star(z)-2\mathbf{x}_j(s)\rangle\bigr|.
\end{align*}
Using \pref{asm: cvx} and \pref{asm: domain},
\begin{equation*}
\bigl|\langle \bg_j(s),\hat{\mathbf{x}}_j(z)+\mathbf{x}_j^\star(z)-2\mathbf{x}_j(s)\rangle\bigr| \le
\|\bg_j(s)\|_2\,
\|\hat{\mathbf{x}}_j(z)+\mathbf{x}_j^\star(z)-2\mathbf{x}_j(s)\|_2
\le 2LGD,
\end{equation*}
and therefore
\begingroup\medmuskip=2mu\thickmuskip=3mu
\begin{align*}
\frac12\|\hat{\mathbf{x}}_j(z)-\mathbf{x}_j^\star(z)\|_{A_j(z)}^2
+
\frac12\|\hat{\mathbf{x}}_j(z)-\mathbf{x}_j^\star(z)\|_{A_{j}(z-3)}^2 \le
\left(1+\beta GD\right)
\sum_{s=z-2}^{z}
\bigl|\langle \bg_j(s),\hat{\mathbf{x}}_j(z)-\mathbf{x}_j^\star(z)\rangle\bigr|.
\end{align*}
\endgroup
Applying Cauchy--Schwarz in the pair of dual norms induced by \(A_j(z)\),
\[
\bigl|\langle \bg_j(s),\hat{\mathbf{x}}_j(z)-\mathbf{x}_j^\star(z)\rangle\bigr|
\le
\|\bg_j(s)\|_{A_j(z)^{-1}}
\|\hat{\mathbf{x}}_j(z)-\mathbf{x}_j^\star(z)\|_{A_j(z)},
\]
we get
\begin{align*}
&\frac12\|\hat{\mathbf{x}}_j(z)-\mathbf{x}_j^\star(z)\|_{A_j(z)}^2
+
\frac12\|\hat{\mathbf{x}}_j(z)-\mathbf{x}_j^\star(z)\|_{A_{j}(z-3)}^2 \\
&\le
\left(1+\beta GD\right)
\left(
\sum_{s=z-2}^{z}
\|\bg_j(s)\|_{A_j(z)^{-1}}
\right)
\|\hat{\mathbf{x}}_j(z)-\mathbf{x}_j^\star(z)\|_{A_j(z)}.
\end{align*}
Dropping the second nonnegative term on the left-hand side gives
\[
\|\hat{\mathbf{x}}_j(z)-\mathbf{x}_j^\star(z)\|_{A_j(z)}
\le
2\left(1+\beta GD\right)
\sum_{s=z-2}^{z}\|\bg_j(s)\|_{A_j(z)^{-1}}.
\]
Consequently,
\begin{align}
\langle \bg_j(z), \hat{\mathbf{x}}_j(z)-\mathbf{x}_j^\star(z)\rangle
&\le
\|\bg_j(z)\|_{A_j(z)^{-1}}
\|\hat{\mathbf{x}}_j(z)-\mathbf{x}_j^\star(z)\|_{A_j(z)} \nonumber\\
&\le
2\left(1+\beta GD\right)
\|\bg_j(z)\|_{A_j(z)^{-1}}
\sum_{s=z-2}^{z} \|\bg_j(s)\|_{A_j(z)^{-1}}.
\label{eq:block-comp-final}
\end{align}
Summing \pref{eq:block-comp-final} over \(z=1,\dots,T/L\), we obtain
\begin{align}
   \sum_{z=1}^{T/L}\langle \bg_j(z), \hat{\mathbf{x}}_j(z)-\mathbf{x}_j^\star(z)\rangle
&\le
2\left(1+\beta GD\right)
\sum_{z=1}^{T/L}
\left(
\|\bg_j(z)\|_{A_j(z)^{-1}}
\sum_{s=z-2}^{z}\|\bg_j(s)\|_{A_j(z)^{-1}}
\right) \notag\\
&\le
\left(1+\beta GD\right)
\sum_{z=1}^{T/L}
\left(
3\|\bg_j(z)\|_{A_j(z)^{-1}}^2
+
\sum_{s=z-2}^{z}\|\bg_j(s)\|^2_{A_j(z)^{-1}}
\right)
\tag{AM-GM inequality}\\
&\le
\left(1+\beta GD\right)
\sum_{z=1}^{T/L}
\left(
3\|\bg_j(z)\|_{A_j(z)^{-1}}^2
+
\sum_{s=z-2}^{z}\|\bg_j(s)\|^2_{A_j(s)^{-1}}
\right)
\tag{$A_j(s) \preceq A_j(z)$}\\
&\le
6\left(1+\beta GD\right)
\sum_{z=1}^{T/L}
\|\bg_j(z)\|_{A_j(z)^{-1}}^2 \notag\\
&\le
\frac{6d\left(1+\beta GD\right)L}{\beta}
\log{\left(
1+ \frac{\beta TG^2}{\eta d}
\right)}.
\label{eq:delayons_2}
\end{align}
where the last inequality holds by using \citet[Lemma~7.29]{orabona2026}.

Substituting \pref{eq: delayons_1} and  \pref{eq:delayons_2} in to \pref{eq:delayons_dec}, we have 
\begin{align}
        &\sum_{z=1}^{T/L}\left( \left\langle \mathbf{g}_j(z), \hat{\mathbf{x}}(z)-\mathbf{x}\right\rangle-\frac{\beta}{2L}
\inner{\bg_j(z), \mathbf{x}_j(z)- \mathbf{x}}^2\right) \notag\\
&\leq \psi_{T/L,j}^\star(\mathbf{x})  -\sum_{z=1}^{T/L} \frac{\beta}{2L}
\inner{\bg_j(z), \mathbf{x}_j(z)-\mathbf{x}}^2 +  \frac{6d\left(1+\beta GD\right)L}{\beta} \log{\left(1+ \frac{\beta TG^2}{\eta d}\right)} \notag\\
&= \frac{\eta}{2}\|\mathbf{x}\|^2_2+  \frac{6d\left(1+\beta GD\right)L}{\beta} \log{\left(1+ \frac{\beta TG^2}{\eta d}\right)} \notag\\
&=  \mathcal{O} \left (\eta\|\mathbf{x}\|^2_2+  \frac{d\left(1+\beta GD\right)L}{\beta} \log{\left(1+ \frac{\beta TG^2}{\eta d}\right)} \right).
        \end{align}
\end{proof}

We now prove \pref{thm:upper-exp} by combining the regret decomposition with the
delayed ONS bound.

\mainexpconcave*

\begin{proof}
Fix any agent \(i\in[N]\) and any comparator \(\mathbf{x}\in\mathcal X\).
By \pref{lem:regret-decomposition},
\[
\Reg_{T,i}
\le
\sum_{z=1}^{T/L}\sum_{j=1}^N
\left(
\left\langle \mathbf{g}_j(z),\widehat{\mathbf{x}}_j(z)-\mathbf{x}\right\rangle
-
\frac{\beta}{2L}
\left\langle \mathbf{g}_j(z),\mathbf{x}_j(z)-\mathbf{x}\right\rangle^2
\right)
+
\mathcal{O}\left(
\frac{GD}{NT}
+
\frac{\beta G^2D^2}{NT}
\right).
\]
Applying \pref{lem:block-ons-exp} to each \(j\in[N]\) gives
\[
\Reg_{T,i}(\mathbf{x})
=
\mathcal{O}\left(
\eta N\|\mathbf{x}\|_2^2
+
\frac{Nd(1+\beta GD)L}{\beta}
\log\left(
1+\frac{\beta G^2T}{\eta d}
\right)
\right).
\]
This proves \pref{eqn:reg-exp-result}. Substituting the choice of \(L\) in
\pref{eqn:para-choice} gives the stated dependence on the spectral gap.
\end{proof}

\subsection{Omitted Details in \pref{sec: lower_exp}}\label{app: lower_exp}
We now turn to the proof of the lower bound in \pref{thm:lowerbound}. For completeness, we restate the lower-bound theorem and provide its proof.
\mainlowerbound*
\begin{proof}
We prove the lower bound on a $1$-connected cycle graph. By definition, we have $K=\frac{N-2}{4}$.
Let
$$
    Z=\left\lfloor \frac{T-1}{K}\right\rfloor,
    \qquad
    c_i=iK,\quad i=0,\ldots,Z,
    \qquad
    c_{Z+1}=T.
$$
We divide the horizon into intervals
$$
    [c_0+1,c_1],\ [c_1+1,c_2],\ \ldots,\ [c_{Z-1}+1,c_Z],
$$
each of length $K$, and one remaining final interval
$
    [c_Z+1,c_{Z+1}].
$

The loss functions for each agent $n\in[N]$ at each $t\in[c_i+1,c_{i+1}]$ are as follows. Specifically, we choose $h_0,\ldots,h_{Z-1}$ as the wealth function in portfolio optimization defined in ~\cite[Theorem~1]{ordentlich1998cost} with horizon $n=Z$ but scaled by $\frac{1}{\alpha}$. Then, for each $i=0,\ldots,Z-1$ and each
$t\in[c_i+1,c_{i+1}]$, define the local losses on the cycle as follows:
\begin{align*}
    &f_{t,1}(x)=\cdots=f_{t,K}(x)=0,\\
    &f_{t,K+1}(x)=\cdots=f_{t,N-K+1}(x)=h_i(x),\\
    &f_{t,N-K+2}(x)=\cdots=f_{t,N}(x)=0.
\end{align*}
For the final interval $t\in[c_Z+1,c_{Z+1}]$, set $f_{t,j}(x)=0$ for every $j\in[N]$. Direct calculation shows that each local loss function is convex and $\alpha$-exp-concave, since each local loss is
either one of the hard losses $h_i$ or the zero function.

For every $i=0,\ldots,Z-1$ and every
$t\in[c_i+1,c_{i+1}]$, the global loss is
$$
    f_t(x)
    =
    \sum_{j=1}^N f_{t,j}(x)
    =
    S h_i(x),
$$
where $S = N-2K+1=2K+3$. In terms of $N$, since $K=(N-2)/4$, we know that
$$
    SK
    =
    \frac{N-2}{4}\left(\frac{N-2}{2}+3\right)
    =
    \frac{(N-2)(N+4)}{8}
    \ge
    \frac{N^2}{8},
$$
for every $N\ge 4$.

Now consider agent $1$. On the $1$-connected cycle, every node $n\in \{K+1,\ldots,N-K+1\}$
is at graph distance at least $K$ from agent $1$. Since the length of $[c_i+1,c_{i+1}]$ is exactly $K$, agent $1$ cannot receive information generated by the
current hidden loss $h_i$ before making the decisions $x_1(c_i+1),\ldots,x_1(c_{i+1})$.

Now we define the averaged decision of agent $1$ during interval $[c_i+1,c_{i+1}]$ as $
    \overline x_1(i)=\frac1K\sum_{t=c_i+1}^{c_{i+1}} x_1(t).$

Since the decision set $\calX$ is convex, we have $\overline x_1(i)\in\calX$.
Moreover, $\overline x_1(i)$ is a valid online decision against the
centralized loss sequence $h_0,\ldots,h_{i-1}$ at round $i$, since $\overline x_1(i)$ does not have the knowledge of $h_i$. By convexity of $h_i$, we know that
$$
    h_i(\overline x_1(i))
    =
    h_i\left(
        \frac1K\sum_{t=c_i+1}^{c_{i+1}}x_1(t)
    \right)
    \le
    \frac1K\sum_{t=c_i+1}^{c_{i+1}}h_i(x_1(t)).
$$
Multiplying both sides by $SK$ and summing over $i=0,\ldots,Z-1$, we obtain that
$$
    SK \sum_{i=0}^{Z-1}h_i(\overline x_i)
    \le
    \sum_{t=c_i+1}^{c_{i+1}} \sum_{i=0}^{Z-1}S h_i(x_1(t)) = \sum_{t=c_i+1}^{c_{i+1}} \sum_{i=0}^{Z-1}f_t(x_1(t)) = \sum_{t=1}^Tf_t(x_1(t)).
$$

On the other hand, for any comparator $x\in\calX$, we have
$$
    \sum_{t=1}^T f_t(x)
    =
    \sum_{i=0}^{Z-1}
    \sum_{t=c_i+1}^{c_{i+1}} S h_i(x)
    =
    SK\sum_{i=0}^{Z-1}h_i(x).
$$
Thus
$\min_{x\in\mathcal K}\sum_{t=1}^T f_t(x)
    =
    SK\min_{x\in\mathcal K}\sum_{i=0}^{Z-1}h_i(x).$
Combining the two displays, we get
$$
\begin{aligned}
    \Reg_{T,1}
    =
    \sum_{t=1}^T f_t(x_1(t))
    -
    \min_{x\in\mathcal K}\sum_{t=1}^T f_t(x) \ge
    SK
    \left(
        \sum_{i=0}^{Z-1}h_i(\overline x_i)
        -
        \min_{x\in\mathcal K}\sum_{i=0}^{Z-1}h_i(x)
    \right).
\end{aligned}
$$

Taking expectation over the sampled losses and any internal randomness of the
distributed algorithm, we know that the loss sequence $h_0,\dots,h_{Z-1}$ defined in~\citet{ordentlich1998cost},
$$
    \mathbb E[R_{T,1}]
    \ge
    SK\cdot \Omega\left(\frac {d}{ \alpha} \log(Z)\right).
$$
Since $SK\ge N^2/8$, this implies that
$
    \mathbb E[R_{T,1}]
    \ge
    \Omega\left(\frac{N^2d}{\alpha}\log(Z)\right)=\Omega(\frac{N^2d}{\alpha}\log(T/N)).
$
Finally, for the $1$-connected cycle graph with $N$, 
communication matrix $P= I-\frac{1}{2(d_{\max}+1)}L_{\mathcal{G}}$,  \cite[Lemma~6]{wan2024nearly} shows that
$
    \frac{\pi^2}{1-\sigma_2(P)}
    \le
    4N^2,
$
implying that
$$
    \mathbb E[R_{T,1}]
    \geq 
    \Omega\left(
        \frac{Nd}{\alpha\sqrt{1-\sigma_2(P)}}
        \log(T/N)
    \right).
$$
This completes the proof.
\end{proof}

\section{Omitted Details from \pref{sec:dolr}}
\label{app:appdolr}

In this section, we show the omitted details in \pref{sec:dolr}. For completeness, we restate the theorem and provide its proof.

\theoremlr*

\begin{proof}
For each block $z$ and round $t\in\mathcal T_z$, we define the reference decision $\mathbf x^\star(z,t)$ as the minimizer of
\[
F_{z,t}^\star(\mathbf x)=\frac\lambda2\|\mathbf x\|_2^2+\sum_{s\in\mathcal S_z}\sum_{\tau\in\mathcal T_s}\ell_\tau(\mathbf x)+\sum_{\substack{\tau\in\mathcal T_z\\\tau\le t}}\ell_\tau(\mathbf x),
\qquad \ell_t(\mathbf x)=\sum_{j=1}^N f_{t,j}(\mathbf x).
\]
This is the reference predictor defined in the proof sketch: it represents the ridge-regression decision with full knowledge of all observations assigned to the same copy through the end of round $t$.

With the definition of $\mathbf x^\star(z,t)$, we decompose the regret as follows:
\begin{align}\label{eqn:OLR-regret-decomp}
\Reg_{T,i}^{\mathrm{OLR}}(\mathbf x)
&=\sum_{z=1}^{\lceil T/L\rceil}\sum_{t\in\mathcal T_z}\sum_{j=1}^N\left[\widetilde f_{t,j}^{\,i}(z,t)-f_{t,j}(\mathbf x)\right]\nonumber\\
&=\sum_{z=1}^{\lceil T/L\rceil}\sum_{t\in\mathcal T_z}\sum_{j=1}^N\left[f_{t,j}(\mathbf x^\star(z,t))-f_{t,j}(\mathbf x)\right]\nonumber\\
&\quad+\sum_{z=1}^{\lceil T/L\rceil}\sum_{t\in\mathcal T_z}\sum_{j=1}^N\left[\widetilde f_{t,j}^{\,i}(z,t)-f_{t,j}(\mathbf x^\star(z,t))\right],
\end{align}
where $\widetilde f_{t,j}^{\,i}(z,t)\triangleq\frac12(\widehat y_{t,j}^{\,i}(z,t)-y_{t,j})^2$ and
$\widehat y_{t,j}^{\,i}(z,t)\triangleq\Pi_{[-Y,Y]}(\langle\mathbf h_{t,j},\mathbf x_i(z,t)\rangle)$.

Consider the first summation in \pref{eqn:OLR-regret-decomp}. Within each copy, the reference decisions minimize the cumulative squared loss plus $\lambda\|\mathbf x\|_2^2/2$. Applying \pref{lem:btrl} separately to the two subsequences, with this regularizer after initialization and zero initial regularizer, gives
\begin{equation}\label{eq:btl}
\sum_{z=1}^{\lceil T/L\rceil}\sum_{t\in\mathcal T_z}\sum_{j=1}^N\left[f_{t,j}(\mathbf x^\star(z,t))-f_{t,j}(\mathbf x)\right]\le\lambda\|\mathbf x\|_2^2.
\end{equation}
Consequently,
\begin{equation}\label{eqn:OLR-regret-decomp2}
\Reg_{T,i}^{\mathrm{OLR}}(\mathbf x)\le\lambda\|\mathbf x\|_2^2+\sum_{z=1}^{\lceil T/L\rceil}\sum_{t\in\mathcal T_z}\sum_{j=1}^N\left[\widetilde f_{t,j}^{\,i}(z,t)-f_{t,j}(\mathbf x^\star(z,t))\right].
\end{equation}

We next explain the role of clipping before controlling the delay and gossip terms. For notational convenience, denote
\[
p_{t,j}^{\,i}(z,t)\triangleq\langle\mathbf h_{t,j},\mathbf x_i(z,t)\rangle,
\qquad
p_{t,j}^{\star}(z,t)\triangleq\langle\mathbf h_{t,j},\mathbf x^\star(z,t)\rangle.
\]
For $|y|\le Y$, define the clipped scalar loss $\phi_y(u)\triangleq\frac12(\Pi_{[-Y,Y]}(u)-y)^2$.
Since $y\in[-Y,Y]$, projection onto $[-Y,Y]$ does not increase the squared distance to $y$, and hence
\[
\phi_y(u)\le\frac12(u-y)^2,\qquad \forall u\in\mathbb R.
\]
Moreover, $\phi_y$ is $2Y$-Lipschitz: for all $u,v\in\mathbb R$,
\begin{align*}
|\phi_y(u)-\phi_y(v)|
&=\frac12\left|\Pi_{[-Y,Y]}(u)-\Pi_{[-Y,Y]}(v)\right|\left|\Pi_{[-Y,Y]}(u)+\Pi_{[-Y,Y]}(v)-2y\right|\\
&\le\frac12|u-v|\cdot4Y=2Y|u-v|.
\end{align*}
Therefore, we have
\begin{align}\label{eqn:relaxY}
\widetilde f_{t,j}^{\,i}(z,t)-f_{t,j}(\mathbf x^\star(z,t))
&=\phi_{y_{t,j}}(p_{t,j}^{\,i}(z,t))-\frac12(p_{t,j}^{\star}(z,t)-y_{t,j})^2\nonumber\\
&\le\phi_{y_{t,j}}(p_{t,j}^{\,i}(z,t))-\phi_{y_{t,j}}(p_{t,j}^{\star}(z,t))\nonumber\\
&\le2Y|p_{t,j}^{\,i}(z,t)-p_{t,j}^{\star}(z,t)|\nonumber\\
&=2Y\left|\langle\mathbf h_{t,j},\mathbf x_i(z,t)-\mathbf x^\star(z,t)\rangle\right|.
\end{align}
Thus, the squared-loss comparison does not require a bound on either model or its raw prediction. We will use this Lipschitz property for the gossip error. For the delay of both sufficient statistics, we retain the quadratic information in the regularized objectives, as detailed next.

Define the auxiliary exact delayed predictor by
\begin{equation}\label{eqn:defXprime}
\mathbf x'(z,t)=A(z)^{-1}b(z),\qquad
A(z)=\lambda I_d+\sum_{s\in\mathcal S_z}\sum_{j=1}^N C_j^0(s),\qquad
b(z)=\sum_{s\in\mathcal S_z}\sum_{j=1}^N Y_j^0(s).
\end{equation}
It uses exactly the blocks available to the corresponding copy, with exact averaging in place of gossip. In particular, it is fixed throughout block $z$. We separate the regret into the regret of this exact delayed predictor and its discrepancy from the decentralized predictor:
\begin{align}\label{eq:clipped-lr-split}
\Reg_{T,i}^{\mathrm{OLR}}(\mathbf x)
&=\sum_{z=1}^{\lceil T/L\rceil}\sum_{t\in\mathcal T_z}\sum_{j=1}^N\left[\phi_{y_{t,j}}(\langle\mathbf h_{t,j},\mathbf x'(z,t)\rangle)-f_{t,j}(\mathbf x)\right]\nonumber\\
&\quad+\sum_{z=1}^{\lceil T/L\rceil}\sum_{t\in\mathcal T_z}\sum_{j=1}^N\left[\widetilde f_{t,j}^{\,i}(z,t)-\phi_{y_{t,j}}(\langle\mathbf h_{t,j},\mathbf x'(z,t)\rangle)\right].
\end{align}
The first term accounts for the delay of the current block. By the clipping argument above, the second term is at most
\begin{equation}\label{eq: clipped_lr_decom}
2Y\sum_{z=1}^{\lceil T/L\rceil}\sum_{t\in\mathcal T_z}\sum_{j=1}^N\left|\langle\mathbf h_{t,j},\mathbf x_i(z,t)-\mathbf x'(z,t)\rangle\right|.
\end{equation}

We first consider the exact delayed predictor. Fix a block $z$ and its assigned copy, and define the regularized objectives immediately before and after this block by
\[
F_z^-(\mathbf x)=\frac\lambda2\|\mathbf x\|_2^2+\sum_{s\in\mathcal S_z}\sum_{\tau\in\mathcal T_s}\ell_\tau(\mathbf x),
\qquad
F_z^+(\mathbf x)=F_z^-(\mathbf x)+\sum_{t\in\mathcal T_z}\ell_t(\mathbf x).
\]
Let $M_z^-\triangleq\min_{\mathbf x}F_z^-(\mathbf x)$ and $M_z^+\triangleq\min_{\mathbf x}F_z^+(\mathbf x)$. The minimizer of $F_z^-$ is $\mathbf x'(z,t)$, whereas that of $F_z^+$ is the reference predictor at the last round of the block. Successive minimum increments telescope within each copy. Since each final minimum is at most the final objective evaluated at any comparator $\mathbf x$, we have
\begin{equation}\label{eq:lr-minimum-telescope}
\sum_{z=1}^{\lceil T/L\rceil}(M_z^+-M_z^-)\le\lambda\|\mathbf x\|_2^2+\sum_{t=1}^T\sum_{j=1}^N f_{t,j}(\mathbf x).
\end{equation}
This is the block form of the centralized regularized-leader comparison. Retaining these minimum increments lets us account for the delay in the covariance as well as in the labels.

To bound the loss within a block, stack its $N|\mathcal T_z|$ feature vectors into the rows of $B_z$, and its labels into $\mathbf y_z$. Completing the square around the minimizer $\mathbf x'(z,t)$ gives
\begin{align*}
M_z^+-M_z^-
&=\min_{\mathbf v}\left\{\frac12\|\mathbf v\|_{A(z)}^2+\frac12\|B_z\mathbf x'(z,t)+B_z\mathbf v-\mathbf y_z\|_2^2\right\}\\
&=\frac12(B_z\mathbf x'(z,t)-\mathbf y_z)^\top\bigl(I+B_zA(z)^{-1}B_z^\top\bigr)^{-1}(B_z\mathbf x'(z,t)-\mathbf y_z).
\end{align*}
The second equality follows by minimizing the quadratic and applying the matrix-inversion identity. Since the smallest eigenvalue of the inverse matrix is $1/(1+\|B_zA(z)^{-1}B_z^\top\|_{\mathrm{op}})$, and clipping cannot increase the squared loss,
\[
M_z^+-M_z^-\ge
\frac{\displaystyle\sum_{t\in\mathcal T_z}\sum_{j=1}^N\phi_{y_{t,j}}(\langle\mathbf h_{t,j},\mathbf x'(z,t)\rangle)}{1+\|B_zA(z)^{-1}B_z^\top\|_{\mathrm{op}}}.
\]
Each clipped squared loss is at most $2Y^2$, and the block contains at most $LN$ examples. Consequently,
\begin{align}\label{eq:lr-block-potential}
&\sum_{t\in\mathcal T_z}\sum_{j=1}^N\phi_{y_{t,j}}(\langle\mathbf h_{t,j},\mathbf x'(z,t)\rangle)-(M_z^+-M_z^-)\nonumber\\
&\le2Y^2LN\frac{\|B_zA(z)^{-1}B_z^\top\|_{\mathrm{op}}}{1+\|B_zA(z)^{-1}B_z^\top\|_{\mathrm{op}}}\nonumber\\
&\le2Y^2LN\log\!\left(\frac{\det(A(z)+B_z^\top B_z)}{\det A(z)}\right).
\end{align}
For the last inequality, use the nonnegative eigenvalues of $B_zA(z)^{-1}B_z^\top$ and the determinant identity $\det(I+B_zA(z)^{-1}B_z^\top)=\det(A(z)+B_z^\top B_z)/\det A(z)$.

It remains to sum the log-determinant increases. Within each copy, $A(z)+B_z^\top B_z$ is precisely the covariance matrix before its next assigned block, so these increases telescope. The trace of either copy's final unregularized covariance is at most $NTH^2$. The arithmetic--geometric mean inequality for the eigenvalues therefore gives, for each copy, a total log-determinant increase at most $d\log(1+NTH^2/(\lambda d))$. This is the elliptical-potential argument at block level; the factor $LN$ accounts for the examples predicted before the copy receives its next statistics. Combining both copies with \pref{eq:lr-minimum-telescope} and \pref{eq:lr-block-potential} yields
\begin{equation}\label{eq:olr_delay}
\sum_{z=1}^{\lceil T/L\rceil}\sum_{t\in\mathcal T_z}\sum_{j=1}^N\left[\phi_{y_{t,j}}(\langle\mathbf h_{t,j},\mathbf x'(z,t)\rangle)-f_{t,j}(\mathbf x)\right]\le\lambda\|\mathbf x\|_2^2+4Y^2LNd\log\left(1+\frac{NTH^2}{\lambda d}\right).
\end{equation}

We now consider the gossip term in \pref{eq: clipped_lr_decom}. Recall that
\[
\widehat A_i(z)=\lambda I_d+N\sum_{s\in\mathcal S_z}\Pi_{\mathbb S_+^d}(C_i(s)),
\qquad \widehat b_i(z)=N\sum_{s\in\mathcal S_z}Y_i(s),
\qquad \mathbf x_i(z,t)=\widehat A_i(z)^{-1}\widehat b_i(z).
\]
Both $\widehat A_i(z)$ and $A(z)$ dominate $\lambda I_d$. Subtracting the exact delayed predictor and applying the matrix-inverse identity gives
\begin{align*}
\mathbf x_i(z,t)-\mathbf x'(z,t)
&=\widehat A_i(z)^{-1}(\widehat b_i(z)-b(z))+
\bigl(\widehat A_i(z)^{-1}-A(z)^{-1}\bigr)b(z)\\
&=\widehat A_i(z)^{-1}(\widehat b_i(z)-b(z))-
\widehat A_i(z)^{-1}(\widehat A_i(z)-A(z))A(z)^{-1}b(z).
\end{align*}
Therefore, by Cauchy--Schwarz,
\begin{equation}\label{eq:lr-inverse-perturbation}
|\langle\mathbf h_{t,j},\mathbf x_i(z,t)-\mathbf x'(z,t)\rangle|
\le\|\mathbf h_{t,j}\|_2\left(\frac{\|\widehat b_i(z)-b(z)\|_2}{\lambda}+\frac{\|\widehat A_i(z)-A(z)\|_{\mathrm{op}}\|b(z)\|_2}{\lambda^2}\right).
\end{equation}
Compared with the label-statistic calculation, the second term accounts for gossiping the covariance matrix.

By \pref{asm:lr}, $|y_{\tau,i}|\le Y$ and $\|\mathbf h_{\tau,i}\|_2\le H$. Hence,
\[
\|Y_i^0(s)\|_2\le\sum_{\tau\in\mathcal T_s}|y_{\tau,i}|\|\mathbf h_{\tau,i}\|_2\le LYH,
\qquad
\|C_i^0(s)\|_F\le\sum_{\tau\in\mathcal T_s}\|\mathbf h_{\tau,i}\mathbf h_{\tau,i}^{\top}\|_F\le LH^2.
\]
Apply the accelerated-gossip contraction used in the proof of \pref{lem:gossip-close-exp}, also to the vectorized covariance matrices. The choice of $L$ in \pref{eq:lr-block-length} makes the contraction factor at most $1/[T^3N^{7/2}(1+H^2/\lambda)]$. The initial distance to the average is at most $2\sqrt N$ times the displayed bound on an individual statistic. Thus, for every completed block $s$,
\begin{align*}
\left\|Y_i(s)-\frac1N\sum_{j=1}^N Y_j^0(s)\right\|_2&\le\frac{2LYH}{T^3N^3(1+H^2/\lambda)},\\
\left\|C_i(s)-\frac1N\sum_{j=1}^N C_j^0(s)\right\|_F&\le\frac{2LH^2}{T^3N^3(1+H^2/\lambda)}.
\end{align*}
The exact average covariance is positive semidefinite. Non-expansiveness of its metric projection in Frobenius norm therefore gives
\[
\left\|\Pi_{\mathbb S_+^d}(C_i(s))-\frac1N\sum_{j=1}^N C_j^0(s)\right\|_F
\le\left\|C_i(s)-\frac1N\sum_{j=1}^N C_j^0(s)\right\|_F.
\]
All blocks in $\mathcal S_z$ are complete, and $L|\mathcal S_z|\le T$. Summing the errors of completed blocks yields
\begin{align*}
\|\widehat b_i(z)-b(z)\|_2
&\le N\sum_{s\in\mathcal S_z}\left\|Y_i(s)-\frac1N\sum_{j=1}^N Y_j^0(s)\right\|_2
\le\frac{2YH}{T^2N^2(1+H^2/\lambda)},\\
\|\widehat A_i(z)-A(z)\|_{\mathrm{op}}
&\le N\sum_{s\in\mathcal S_z}\left\|\Pi_{\mathbb S_+^d}(C_i(s))-\frac1N\sum_{j=1}^N C_j^0(s)\right\|_F
\le\frac{2H^2}{T^2N^2(1+H^2/\lambda)}.
\end{align*}
Moreover, $\|b(z)\|_2\le NTYH$. Substituting these estimates into \pref{eq:lr-inverse-perturbation}, and using $TN\ge1$, gives
\[
|\langle\mathbf h_{t,j},\mathbf x_i(z,t)-\mathbf x'(z,t)\rangle|
\le\frac{2YH}{\lambda TN}\|\mathbf h_{t,j}\|_2.
\]
Consequently,
\begin{align}\label{eq:lr_cons}
2Y\sum_{z=1}^{\lceil T/L\rceil}\sum_{t\in\mathcal T_z}\sum_{j=1}^N
|\langle\mathbf h_{t,j},\mathbf x_i(z,t)-\mathbf x'(z,t)\rangle|
&\le\frac{4Y^2H}{\lambda TN}\sum_{z=1}^{\lceil T/L\rceil}\sum_{t\in\mathcal T_z}\sum_{j=1}^N\|\mathbf h_{t,j}\|_2\nonumber\\
&\le\frac{4Y^2H^2}{\lambda}.
\end{align}
The last inequality uses $\|\mathbf h_{t,j}\|_2\le H$ and the fact that the blocks partition $[T]$.

Combining \pref{eq:clipped-lr-split}, \pref{eq: clipped_lr_decom}, \pref{eq:olr_delay}, and \pref{eq:lr_cons}, we have
\begin{equation}
\Reg_{T,i}^{\mathrm{OLR}}(\mathbf x)\le\lambda\|\mathbf x\|_2^2+\frac{4Y^2H^2}{\lambda}+4Y^2LNd\log\left(1+\frac{NTH^2}{\lambda d}\right),
\end{equation}
which proves \pref{thm: lr}.
\end{proof}

\end{document}